\documentclass{llncs}
\usepackage{times}
\usepackage{tikz}
\usepackage{graphicx}
\usepackage{latexsym}
\usepackage{eurosym}
\usepackage{amsmath,amssymb,stmaryrd,pifont,wasysym}
\usepackage{dsfont}
\usepackage{array}
\usepackage[ruled,vlined,linesnumbered]{algorithm2e}
\usepackage{pgfplots}
\pgfplotsset{compat=1.14}
\usepackage{enumitem}
\newcommand{\outranks}{\succcurlyeq}

\newcommand{\X}{\mathbb{X}}
\newcommand{\N}{\mathcal{N}}
\newcommand{\T}{\mathcal{T}}
\newcommand{\M}{\mathcal{M}}
\newcommand{\C}{C^1 \prec \dots \prec C^p}
\newcommand{\NCS}{\textsc{NCS}\ }
\newcommand{\frontier}{b}

\def \Na{{cost}}
\def \Nb{{acceleration}}
\def \Nc{{braking}}
\def \Nd{{road\ holding}}

\def \Excellent{$\star \star \star$}
\def \Good{$\star \star$}
\def \Bad{$\star $}

\begin{document}

\title{An efficient SAT formulation for learning multiple criteria non-compensatory sorting rules from examples}

\titlerunning{SAT for NCS}  
%
\author{K. Belahc\`ene\inst{1} \and C. Labreuche\inst{2} \and N. Maudet\inst{3} \and V. Mousseau\inst{1} \and W. Ouerdane\inst{1}
}
\authorrunning{}
%
\tocauthor{}
\institute{LGI, CentraleSup\'elec, Universit\'e Paris-Saclay, Gif-sur-Yvette, France
\and
Thales Research \& Technology, 91767 Palaiseau Cedex, France
\and Sorbonne Universit\'es, UPMC Univ Paris 06, CNRS, LIP6 UMR 7606, 75005 Paris}

\maketitle              

\begin{abstract}
The literature on Multiple Criteria Decision Analysis (MCDA) proposes several methods  in order to sort alternatives evaluated on several attributes into ordered classes. Non Compensatory Sorting models (NCS) assign alternatives to classes based on the way they compare to multicriteria profiles separating the consecutive classes. Previous works have proposed approaches to learn the parameters of a NCS model based on a learning set. Exact approaches based on mixed integer linear programming ensures that the learning set is best restored, but can only handle datasets of limited size. Heuristic approaches can handle large learning sets, but do not provide any guarantee about the inferred model. In this paper, we propose an alternative formulation to learn a NCS model. This formulation, based on a SAT problem, guarantees to find a model fully consistent with the learning set (whenever it exists), and is computationally much more efficient than existing exact MIP approaches. 
\end{abstract}
\section{Introduction}

Multiple Criteria Decision Analysis (MCDA) aims at supporting a decision maker (DM) in making decisions among options described according various points of view, formally represented by monotone functions called \emph{criteria}. In this paper, decisions are modeled as an \emph{ordinal sorting problem}, where alternatives are to be assigned to a class in the  set of predefined ordered classes.

The literature contains several multiple criteria sorting methods which can be distinguished into \textit{(i)} value based sorting methods (see e.g. \cite{MARICHAL2005,GRECO-ET-AL2010,GRECO2011,SOYLU2011})), \textit{(ii)} outranking based sorting methods (see e.g. \cite{bouyssoumarchant2007a,bouyssoumarchant2007b,leroy2011ADT,sobrieADT2013,Zheng-et-al-COR2014,MEYER2017}) and, \textit{(iii)} rule based sorting methods  (see e.g. \cite{gms2002,gms2016}). 

We address the problem of ordinal sorting with an outranking based sorting model: the non compensatory model (NCS, cf.\cite{bouyssoumarchant2007a,bouyssoumarchant2007b}), in which an object need not be at least as good as the profile’s value on all criteria.  NCS assigns an alternative to a category above a profile if it is at least as good as the profile on a \emph{sufficient} coalition of criteria; the family of sufficient coalitions can be any upset of the set of all subsets of criteria. A particular case of NCS occurs when the family of sufficient coalitions of criteria can be defined using additive criteria weights and threshold. The literature refers to this additive case as the MR-Sort model (see e.g. \cite{leroy2011ADT,sobrieADT2013}). Both MR-Sort and NCS models are particular cases of the Electre Tri model, a method for sorting alternatives in ordered categories based on an outranking relation (see \cite{roybouyssou1993}, pp. 389-401 or \cite{bouyssouetal2006}, pp. 381-385).

Our aim is to learn a NCS model from preference information given in the form of a reference assignment. Such approach makes it possible to integrate the decision maker preferences in the model without asking her for the preference parameter values. Such indirect elicitation has been developped for Electre Tri \cite{ms98}, MR-sort \cite{leroy2011ADT}, UTADIS \cite{zd2002}.

\begin{center}
	\begin{algorithm}[ht!]
			\caption{\sc{Learning a model-based classifier}} \label{algo1}
            
			\KwIn{a tuple of criteria, a tuple of ordered categories, a multicriteria sorting model, an assignment of alternatives to categories}
			
			\KwResult{a representation of the assignment in the model, or \emph{None} if the assignment is not representable in the model}
            
            \textbf{encode} the assignment into a formulation $\Phi$
            
            try to \textbf{solve} the formulation $\Phi$
				
               \qquad  \textbf{decode} the solution into a model
                
               \qquad  return the model
               
			{except} NoSolution
            
            	\qquad return \emph{None}

	\end{algorithm}
	
\end{center}

Algorithm \ref{algo1} describes a general framework that has been widely used (see e.g. \cite{UTA,leroy2011ADT})  in order to leverage the power of generic mathematical programming solvers to learn the parameters of a multicriteria sorting procedure from examples. The work flow is divided into three phases: 
 the problem is \emph{encoded} into a formulation, this formulation is passed to an external \emph{solver}, and a solution, if found, is then \emph{decoded} into a model.
The \emph{faithfulness} of this approach is guaranteed if, and only if: \begin{enumerate}
\item the encoded formulation must have a solution as soon as the assignment can be represented in the model;
\item the solver is complete, in the sense that it yields a solution if and only if there is at least one;
\item the decoded model actually represents the assignment.
\end{enumerate}
 
To the authors' knowledge, until now, general NCS models have been deemed too computationally difficult to address with this approach. Restrictions to MR-Sort have been considered, either in \cite{leroy2011ADT} with a mixed integer programming (MIP) formulation, but this approach turned out to be inadequate to handle large datasets, or by \cite{sobrieADT2013,sobrieetal2015ADT} using a metaheuristic solving procedure that handles large datasets but offers no guarantee of its completeness (cf. point 2 above). 
The aim of this paper is to investigate an alternative venue: considering U-NCS, a broad subset of NCS models, that encompasses MR-Sort (for precise definitions of these models, see Section \ref{subsection : ncs models}), and formulating the problem of representing an assignment by a model in U-NCS as a boolean satisfiability problem (SAT). We prove that both the encoding and the decoding satisfy the faithfulness requirements 1 and 3 above. We are thus able to leverage the advances made in the field of generic SAT solvers, to reach unprecedented computational performance in the learning of non compensatory sorting models. 

The paper is organized as follows. In Section \ref{section:problem}, we present the notions and concepts related to the formulation of the problem of learning parameters of a non compensatory sorting model. 
In Section \ref{section : sat}, we develop our binary satisfaction (SAT) problem formulation for inferring a U-NCS model from a learning set, and show it has the desired properties of necessity and sufficiency regarding the representation of an assignment in the U-NCS model. In Section \ref{section : MIP learning MR-SORT}, we recall the bases of using a mixed integer formulation to learn the parameters of a MR-Sort model. After that, we propose experiments to assess the pertinence and interest of this formulation in Section \ref{section:implem}. In Section \ref{section:discussion}, we discuss the obtained results. Finally, in Section \ref{section:conclu},  we conclude by pointing some future interesting perspectives . 

\section{Position of the problem}
\label{section:problem}
In this section, we detail the notions permitting to formulate the problem of learning the parameters of a non compensatory sorting models. We define the vocabulary of ordinal sorting and we formalize the notion of ordinal sorting procedure. We are then able to precisely describe the problem of representing a given assignment in a given ordinal sorting model. We proceed by describing the broad class of non compensatory sorting models, and two narrower subclasses of particular interest, namely U-NCS and MR-Sort. In Section \ref{subsection : ncs}, we formulate the NCS model and define its parameters. In Section \ref{subsection : learning}, we specify the expected inputs and outputs of the learning problem.
\subsection{Vocabulary of multicriteria ordinal sorting}
\label{subsection : ncs}

An ordinal sorting problem consists in assigning a category, taken among a given, finite set of \emph{categories} $\C$ ordered by desirability, to \emph{alternatives} described by several attributes.

We assume $\mathcal{N}$ is a finite set of \emph{criteria}, where each criterion $i\in\mathcal{N}$ maps alternatives to values among an ordered set $(\X_i,\le_i)$, the order relation $\le_i$ meaning ``weakly worse that''\footnote{ This setting differs from the one described by \cite{bouyssoumarchant2007a,bouyssoumarchant2007b}, in the sense that we suppose the attributes describing the alternatives are already sorted by the criteria according to their desirability: here, the order relation on each set $\X_i$ needs not be constructed from holistic preference statements, but is assumed to be established beforehand, e.g. in a previous phase of a decision aiding process structured according to \cite{bouyssouetal2006} (this is often the case in applications).}. Alternatives are thus described by a $|\mathcal{N}|$-tuple of multiple criteria values called \emph{profiles}. We denote $\X=\prod_{i\in\N} \X_i$ the set of all possible profiles-- either describing actual alternatives or virtual ones.


As an analogy with a voting system where criteria would act as voters, subsets of $\N$ are called \emph{coalitions} of criteria. 
The following function maps a pair of profiles to the coalition of criteria  weakly favorable to the former.
$$\begin{array}{lccl} O_\N: & \X \times \X & \longrightarrow & \mathcal{P(N)} \\ & (x,y) & \mapsto & \{i\in\N : x_i \ge_i y_i \} \end{array}$$
When $O_\N(x,y)=\N$, the alternative $x$ is at least as good as the alternative $y$ with respect to all criteria, and we say $x$ \emph{weakly dominates} $y$ in the sense of Pareto. Weak dominance defines a partial order $\Delta$ on the set of profiles $\X$.

In the remainder of this article, we assume the sets of criteria $\N$, of profiles $\X$ and of categories $C$ are given, and we endeavor to define a \emph{sorting procedure}, a non-decreasing mapping $\X$, ordered by dominance to $\C$.

\subsection{Non-compensatory sorting models}
\label{subsection : ncs models}
In \cite{bouyssoumarchant2007a,bouyssoumarchant2007b}, Bouyssou and Marchant define a set of sorting procedures deemed as \emph{non-compensatory}. We briefly recall the definition of the \emph{non compenastory sorting (NCS) model}, as well as two specific subsets of this model, \emph{U-NCS} and \emph{MR-Sort}. 

All these classes of non-compensatory sorting models, rely on the notions of satisfactory values of the criteria and sufficient coalitions of criteria that combines into defining the fitness of an alternative: an alternative is deemed fit if it has satisfactory values on a sufficient coalition of criteria.

This notion is straightforward to implement when there are only two categories: the sufficient coalitions $\T$ form an upset\footnote{An upset is an upward closed subset of an ordered set, i.e. if $b$ is greater than $a$ and $a$ belongs to an upset, then so does $b$.} of the power set of $\N$ and, for each criterion $i\in\N$, the satisfactory values $\mathcal{A}_i \subset \X_i$ form an upset that can be described by its lower bound $b_i\in\X_i$ -- meaning a value is satisfactory if, and only if, it is greater or equal to the threshold $b_i$, thus defining a \emph{limiting profile} $b\in\X$. With more than two categories, the notions of sufficient coalitions and satisfactory values are declined per category -- denoted respectively $\langle \mathcal{A}_i^k \rangle_{i\in\N,k\in[1..p-1]}$ and $\langle \T^k \rangle_{k\in[1..p-1]}$. The ordering of the categories $\C$ translates into a nesting of the sufficient coalitions:
\begin{subequations}
\begin{equation} \forall k\in[1..p-1],\ \T^k \textrm{ is an upset of }(2^\N,\subseteq) \textrm{ and}\ \T^1 \supseteq \dots \supseteq \T^{p-1} 
\label{eq : nested sufficient coalitions}
\end{equation}
and also a nesting of the satisfactory values:
\begin{equation}
\forall i\in\N,\ \forall k\in[1..p-1],\ \mathcal{A}_i^k \textrm{ is an upset of }(\X_i, \le_i) \textrm{ and }\mathcal{A}_i^1 \subseteq \dots \subseteq \mathcal{A}_i^{p-1}
\label{eq : nested satisfactory values}
\end{equation}
Condition \eqref{eq : nested satisfactory values} translates into an ordering of the values $\langle b_i^k \rangle_{k\in[1..p-1]}$ for a given criterion $i\in\N$, or an ordering of the limiting profiles:
\begin{equation}
b^1,\dots,b^{p-1} \textrm{ is a non-decreasing sequence of }(\X,\Delta) 
\label{eq : increasing profiles}
\end{equation}
\end{subequations}
For convenience, these sequences are augmented with trivial elements on both ends: $\T^0 := \mathcal{P}(\N)$, $T^p := \emptyset$, $\forall i\in\N \mathcal{A}_i^0=\X_i,\ \mathcal{A}_i^p=\emptyset$, $b^0:=\bot, b^p:=\top$.

\begin{definition}[Non-compensatory sorting NCS, \cite{bouyssoumarchant2007b}]
Given a set of criteria $\N$ and an ordered set of categories $\C$, for all pairs of tuples $(\langle b \rangle,\langle \T \rangle)$ where $\langle b \rangle$ satisfies \eqref{eq : increasing profiles} and $\langle \T \rangle$ satisfies \eqref{eq : nested sufficient coalitions}, the sorting function $NCS_{(\langle b \rangle,\langle \T \rangle)}$ maps a profile $x\in\X$ to the category $C^k$ such that $O_\N (x,b^k)\in\T^k$ and $O_\N (x,b^{k+1})\notin\T^{k+1}$.
\end{definition}

The set of preference parameters -- all the pairs $(\langle b \rangle,\langle \T \rangle)$ satisfying  \eqref{eq : nested sufficient coalitions} and \eqref{eq : increasing profiles} -- can be considered too wide and too unwieldy for practical use in the context of a decision aiding process. Therefore, following \cite{bouyssoumarchant2007b}, one may consider to restrict either the sequence of limiting profiles, or the sequence of sufficient coalitions. In order to remain compatible with Electre Tri, we elect the latter.

\begin{definition}[Non-compensatory sorting with a unique set of sufficient coalitions U-NCS]
Given a set of criteria $\N$ and an ordered set of categories $\C$, for all pairs $(\langle b \rangle,\T)$ where the tuple $\langle b \rangle$ satisfies \eqref{eq : increasing profiles} and $\T$ is an upset of coalitions, the sorting function \emph{U-NCS}$_{(\langle b \rangle,\T)}$ maps a profile $x\in\X$ to the category $C^k$ such that $O_\N (x,b^k)\in\T$ and $O_\N (x,b^{k+1})\notin\T$.
\end{definition}

A further restriction of U-NCS is of particular interest: in the MR-Sort model, introduced in \cite{leroy2011ADT}, the sufficient coalitions are represented in a compact form which is more amenable to linear programming. As an analogy to a voting setting, each criterion $i\in\N$ may be assigned a \emph{voting power} $w_i \ge 0$ so that a given coalition of criteria $B\subseteq \N$ is deemed sufficient if, and only if, its combined voting power $\sum_{i\in B} w_i$ is greater than a given \emph{qualification threshold} $\lambda$. 

\begin{definition}[majority rule sorting MR-Sort]
Given a set of criteria $\N$, the majority rule \emph{MR} maps a pair $(\langle w \rangle,\lambda)$, where $\langle w \rangle$ is a $|\N|$-tuple of nonnegative real numbers and $\lambda$ a nonnegative real number, to an upset \emph{MR}$(\langle w \rangle,\lambda)$ of the power set of $\N$ defined by the relation:
\begin{equation}
\forall B \subseteq \N,\ B\in \emph{MR}(\langle w \rangle,\lambda) \iff \sum_{i\in B} w_i \ge \lambda
\tag{MR} \label{eq : majority rule}
\end{equation}
Given, in addition, a set of categories $\C$, for all triples $(\langle b \rangle, \langle w \rangle, \lambda)$ where the tuple $\langle b \rangle$ satisfies \eqref{eq : increasing profiles}, $\langle w \rangle$ is a $|\N|$-tuple of nonnegative real numbers and $\lambda$ a nonnegative real number, \emph{MR-Sort}$_{(\langle b \rangle, \langle w  \rangle, \lambda)}$ is the sorting function \emph{U-NCS}$_{(\langle b \rangle,\emph{MR}(\langle w \rangle,\lambda))}$. 
\end{definition}

\begin{example}
Terry is a journalist and prepares a car review for a special issue. He considers a number of popular  car models, and wants to sort them in order to present a sample of cars ``selected for you by the redaction'' to the readers.

This selection is based on 4 criteria : cost (\euro{}), acceleration (measured by the time, in seconds, to reach 100 km.h$^{-1}$ from full stop -- lower is better), braking power and road holding, both measured on a qualitative scale ranging from 1 (lowest performance) to 4 (best performance). The performances of six models are described in Table \ref{tab:PerformanceTable}:
 
{\small
\begin{table}
\begin{center}
\begin{tabular}{|c|c|c|c|c|} \hline \bf {\quad model \quad} & \bf {\quad\Na\quad} & \bf {\quad\Nb\quad} & \bf {\quad\Nc\quad} & \bf {\quad\Nd\quad} \\ \hline
$m_1$ & 16 973 & 29 & 2.66 & 2.5 \\ $m_2$ & 18 342 & 30.7 & 2.33 & 3 \\ $m_3$ & 15 335 & 30.2 & 2 & 2.5 \\ $m_4$ & 18 971 & 28 & 2.33 & 2 \\ $m_5$ & 17 537 & 28.3 & 2.33 & 2.75 \\ $m_6$ & 15 131 & 29.7 & 1.66 & 1.75 \\ \hline 
\end{tabular}
\end{center}
\label{tab:PerformanceTable}
\caption{Performance table}
\end{table}
}
In order to assign these models to a class among $C^{1^\star}$ {(average)} $\prec C^{2^\star}$ {(good)} $\prec C^{3^\star}$ {(excellent)}, Terry considers a NCS model: \begin{itemize}
\item where the values on each criterion are sorted between $1^\star$ (average) and $2^\star$ (good) by the following profiles: $b^{1^\star}_\Na= 17\ 250$, $b^{1^\star}_\Nb=30$, $b^{1^\star}_\Nc=2.2$, $b^{1^\star}_\Nd=1.9$.
The boundary between $2^\star$ and $3^\star$ (excellent) is fixed by the profiles: $b^{2^\star}_\Na= 15\ 500$, $b^{2^\star}_\Nb=28.8$, $b^{2^\star}_\Nc=2.5$, $b^{2^\star}_\Nd=2.6$.
\smallskip

Figure \ref{fig : Terry criteria} and Table \ref{tab : Terry table} depict the performance of the six alternatives.

\begin{table}[ht!]
\begin{center}
\begin{tabular}{|c|c|c|c|c|} \hline \bf {\quad model \quad} & \bf {\quad\Na\quad} & \bf {\quad\Nb\quad} & \bf {\quad\Nc\quad} & \bf {\quad\Nd\quad} \\ \hline
$m_1$ & \Good & \Good & \Excellent & \Good \\ $m_2$ & \Bad & \Bad & \Good & \Excellent \\ $m_3$ & \Excellent & \Bad & \Bad & \Good \\ $m_4$ & \Bad & \Excellent & \Good & \Good \\ $m_5$ & \Bad & \Excellent & \Good & \Excellent \\ $m_6$ & \Excellent & \Good & \Bad & \Bad \\ \hline 

\end{tabular}
\end{center}
\caption{Categorization of performances}
\label{tab : Terry table}
\end{table}

\begin{figure}[ht!]
\begin{tikzpicture}
\tikzstyle{alt-O}=[circle,inner sep = 2.5pt,draw, fill=black!50, line width=0.5pt]
\tikzstyle{alt-X}=[rectangle, inner sep = 3pt, draw, fill=black!50, line width=0.5pt]

\draw[decorate, decoration={brace}, yshift=2.5ex]  (1,0) -- node[above=1ex] {\Bad}  (5.3,0);
\draw[decorate, decoration={brace}, yshift=2.5ex]  (5.7,0) -- node[above=1ex] {\Good}  (8.3,0);
\draw[decorate, decoration={brace}, yshift=2.5ex]  (8.7,0) -- node[above=1ex] {\Excellent}  (11,0);
\draw [->,thick] (1,0) -- (11,0);
\draw (6,-0.2) [thick] -- (6,0.2);
\draw (3.4,-0.2) [thick] -- (3.4,0.2);
\draw (9.2,-0.2) [thick] -- (9.2,0.2);
\draw (2,-0.2) [thick] -- (2,0.2);
\draw (5,-0.2) [thick] -- (5,0.2);
\draw (9.8,-0.2) [thick] -- (9.8,0.2);
\draw (5.5,-0.2) [thick] -- (5.5,0.3);
\draw (8.5,-0.2) [thick] -- (8.5,0.3);



\path
node at (6,-0.5) {$m_1$}
node at (3.4,-0.5) {$m_2$}
node at (9.2,-0.5) {$m_3$}
node at (2,-0.5) {$m_4$}
node at (5,-0.5) {$m_5$}
node at (9.8,-0.5) {$m_6$}
node at (5.6,0.6) {$b^{1^\star}$}
node at (8.6,0.6) {$b^{2^\star}$}
;

\path
node at (11,0) [anchor=west] {\bf \quad \Na}
;

\draw[decorate, decoration={brace}, yshift=2.5ex]  (1,-2) -- node[above=1ex] {\Bad}  (3.8,-2);
\draw[decorate, decoration={brace}, yshift=2.5ex]  (4.2,-2) -- node[above=1ex] {\Good}  (6.4,-2);
\draw[decorate, decoration={brace}, yshift=2.5ex]  (6.8,-2) -- node[above=1ex] {\Excellent}  (11,-2);
\draw [->,thick] (1,-2) -- (11,-2);
\draw (6,-2.2) [thick] -- (6,-1.8);
\draw (2.6,-2.2) [thick] -- (2.6,-1.8);
\draw (3.6,-2.2) [thick] -- (3.6,-1.8);
\draw (8,-2.2) [thick] -- (8,-1.8);
\draw (7.4,-2.2) [thick] -- (7.4,-1.8);
\draw (4.6,-2.2) [thick] -- (4.6,-1.8);
\draw (4,-2.2) [thick] -- (4,-1.7);
\draw (6.6,-2.2) [thick] -- (6.6,-1.7);

\path
node at (6,-2.5) {$m_1$}
node at (2.6,-2.5) {$m_2$}
node at (3.6,-2.5) {$m_3$}
node at (8,-2.5) {$m_4$}
node at (7.4,-2.5) {$m_5$}
node at (4.6,-2.5) {$m_6$}
node at (4.1,-1.4) {$b^{1^\star}$}
node at (6.7, -1.4) {$b^{2^\star}$}
;

\path
node at (11,-2) [anchor=west] {\bf \quad \Nb}
;

\draw[decorate, decoration={brace}, yshift=2.5ex]  (1,-4) -- node[above=1ex] {\Bad}  (4.4,-4);
\draw[decorate, decoration={brace}, yshift=2.5ex]  (4.8,-4) -- node[above=1ex] {\Good}  (5.8,-4);
\draw[decorate, decoration={brace}, yshift=2.5ex]  (6.2,-4) -- node[above=1ex] {\Excellent}  (11,-4);
\draw [->,thick] (1,-4) -- (11,-4);
\draw (6.6,-4.2) [thick] -- (6.6,-3.8);
\draw (5.3,-4.2) [thick] -- (5.3,-3.8);
\draw (4,-4.2) [thick] -- (4,-3.8);
\draw (5.3,-4.2) [thick] -- (5.3,-3.8);
\draw (5.3,-4.2) [thick] -- (5.3,-3.8);
\draw (2.6,-4.2) [thick] -- (2.6,-3.8);
\draw (4.6,-4.2) [thick] -- (4.6,-3.7);
\draw (6,-4.2) [thick] -- (6,-3.7);

\path
node at (6.6,-4.5) {$m_1$}
node at (4,-4.5) {$m_3$}
node at (5.3,-4.3) [anchor = north, text width = 3.6em, align=center]{$m_4$, $m_5$, $m_2$}
node at (2.6,-4.5) {$m_6$}
node at (4.7,-3.4) {$b^{1^\star}$}
node at (6.1,-3.4) {$b^{2^\star}$}
;

\path
node at (11,-4) [anchor=west] {\bf \quad \Nc}
;

\draw[decorate, decoration={brace}, yshift=2.5ex]  (1,-6) -- node[above=1ex] {\Bad}  (2.7,-6);
\draw[decorate, decoration={brace}, yshift=2.5ex]  (3.1,-6) -- node[above=1ex] {\Good}  (5.2,-6);
\draw[decorate, decoration={brace}, yshift=2.5ex]  (5.6,-6) -- node[above=1ex] {\Excellent}  (11,-6);
\draw [->,thick] (1,-6) -- (11,-6);
\draw (5,-6.2) [thick] -- (5,-5.8);
\draw (6.7,-6.2) [thick] -- (6.7,-5.8);
\draw (5,-6.2) [thick] -- (5,-5.8);
\draw (3.3,-6.2) [thick] -- (3.3,-5.8);
\draw (5.8,-6.2) [thick] -- (5.8,-5.8);
\draw (2.5,-6.2) [thick] -- (2.5,-5.8);
\draw (2.9,-6.2) [thick] -- (2.9,-5.7);
\draw (5.4,-6.2) [thick] -- (5.4,-5.7);

\path
node at (5,-6.3) [anchor = north, text width = 2em]{$m_1$ $m_3$}
node at (6.7,-6.5) {$m_2$}
node at (3.3,-6.5) {$m_4$}
node at (5.8,-6.5) {$m_5$}
node at (2.5,-6.5) {$m_6$}
node at (3.0,-5.4) {$b^{1^\star}$}
node at (5.5,-5.4) {$b^{2^\star}$}
;

\path
node at (11,-6) [anchor=west] {\bf \quad \Nd}
;
\end{tikzpicture}
\caption{Representation of performances w.r.t. catergory limits}
\label{fig : Terry criteria}
\end{figure}
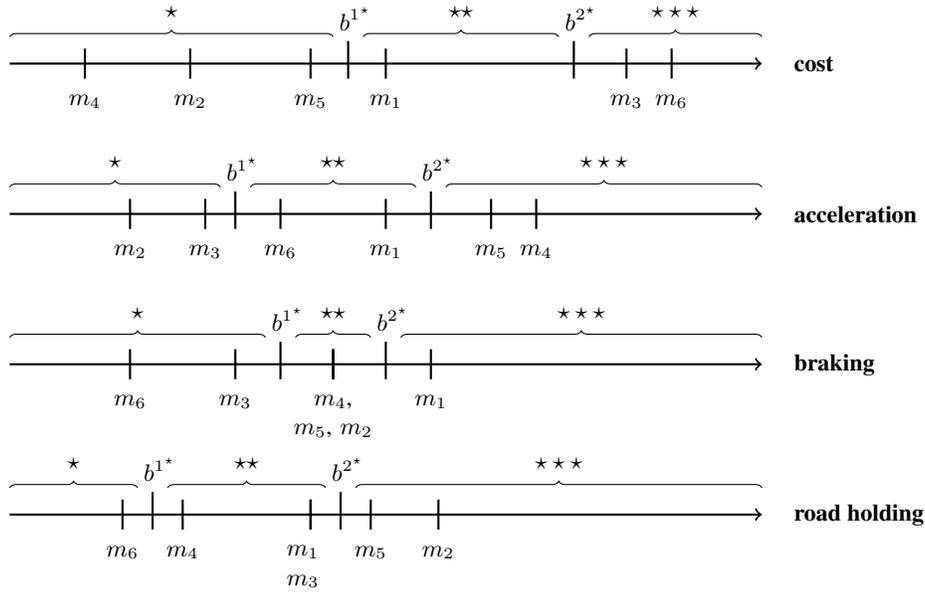

\end{itemize}
\begin{itemize}
\item These appreciations are then aggregated by the following rule: \emph{an alternative is categorized good or excellent if it is good or excellent on \Na\ or \Nb, and good or excellent on \Nc\ or \Nd. It is categorized excellent if it is excellent on \Na\ or \Nb, and excellent on \Nc\ or \Nd}. Being excellent on some criterion does not really help to be considered good overall, as expected from a non compensatory model. Sufficient coalitions are represented on Figure \ref{fig:coalitions}.
\end{itemize}

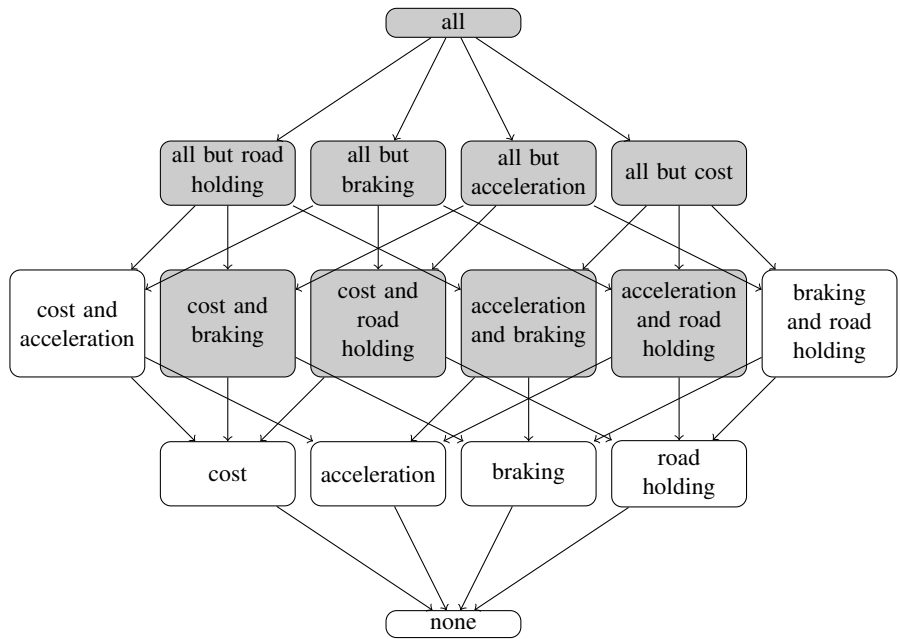
\begin{figure}[ht!]

\begin{center}
\begin{tikzpicture}

\tikzstyle{suff}=[rectangle, fill=black!20, text width=5em, text height= 1.2ex, align = center, rounded corners, draw]
\tikzstyle{insuff}=[rectangle, rounded corners, text width=5em, draw,
align = center]
\tikzstyle{sup}=[->]

\path (-5,0) node(AB) [insuff, minimum height = 10ex] {\Na\ and \Nb}
(-3,0) node(AC) [suff, minimum height = 10ex] {\Na\ and \Nc}
(-1,0) node(AD) [suff, minimum height = 10ex] {\Na\ and \Nd}
(1,0) node(BC) [suff, minimum height = 10ex] {\Nb\ and \Nc}
(3,0) node(BD) [suff, minimum height = 10ex] {\Nb\ and \Nd}
(5,0) node(CD) [insuff, minimum height = 10ex] {\Nc\ and \Nd}
(-3,2) node(ABC) [suff, minimum height = 6ex] {all but \Nd}
(-1,2) node(ABD) [suff, minimum height = 6ex] {all but \Nc}
(1,2) node(ACD) [suff, minimum height = 6ex] {all but \Nb}
(3,2) node(BCD) [suff, minimum height = 6ex] {all but \Na}
(0,4) node(ABCD) [suff] {all}
(-3,-2) node(A) [insuff, minimum height = 6ex] {\Na}
(-1,-2) node(B) [insuff, minimum height = 6ex] {\Nb}
(1,-2) node(C) [insuff, minimum height = 6ex] {\Nc}
(3,-2) node(D) [insuff, minimum height = 6ex] {\Nd}
(0,-4) node(vide) [insuff] {none}
;
\draw (ABCD) [sup] -- (ABC);
\draw (ABCD) [sup] -- (ABD);
\draw (ABCD) [sup] -- (ACD);
\draw (ABCD) [sup] -- (BCD);

\draw (ABC) [sup] -- (AB);
\draw (ABC) [sup] -- (AC);
\draw (ABC) [sup] -- (BC);
\draw (ABD) [sup] -- (AB);
\draw (ABD) [sup] -- (AD);
\draw (ABD) [sup] -- (BD);
\draw (ACD) [sup] -- (AD);
\draw (ACD) [sup] -- (AC);
\draw (ACD) [sup] -- (CD);
\draw (BCD) [sup] -- (BC);
\draw (BCD) [sup] -- (BD);
\draw (BCD) [sup] -- (CD);
\draw (AB) [sup] -- (A);
\draw (AB) [sup] -- (B);
\draw (AC) [sup] -- (A);
\draw (AC) [sup] -- (C);
\draw (AD) [sup] -- (A);
\draw (AD) [sup] -- (D);
\draw (BC) [sup] -- (B);
\draw (BC) [sup] -- (C);
\draw (BD) [sup] -- (B);
\draw (BD) [sup] -- (D);
\draw (CD) [sup] -- (C);
\draw (CD) [sup] -- (D);
\draw (A) [sup] -- (vide);
\draw (B) [sup] -- (vide);
\draw (C) [sup] -- (vide);
\draw (D) [sup] -- (vide);

\end{tikzpicture}

\caption{Sufficient (grey) and insufficient (white) coalitions of criteria. Arrows denote strength - pointing towards the weaker}
\label{fig:coalitions}
\end{center}
\end{figure}
\noindent Finally, the model yields the following assignments (Table \ref{tab:ModelAssignments}):

\begin{table}[ht!]
\begin{center}
\begin{tabular}{|c|c|} \hline \bf {\quad model \quad} & \bf {\quad assignment \quad} \\ \hline
$m_1$ & \Good \\ $m_2$ & \Bad \\ $m_3$ & \Good \\ $m_4$ & \Good \\ $m_5$ & \Excellent \\ $m_6$ & \Bad \\ \hline 
\end{tabular}
\end{center}
\caption{Model Assignments}
\label{tab:ModelAssignments}
\end{table}

\end{example}

\subsection{The disaggregation paradigm: learning preference parameters from assignment examples}
\label{subsection : learning}


For a given decision situation, assuming the NCS model is relevant to structure the DM's preferences, what parameters should be selected to fully specify the NCS model that corresponds to the DM viewpoint? An option would be to simply ask the decision maker to describe, to her best knowledge, the limit profiles between class and to enumerate the minimal sufficient coalitions. In order to get this information as quickly and reliably as possible, an analyst could make good use of the \emph{model-based elicitation strategy} described in \cite{belahcene-et-alEJORsubmitted}, as it permits to  obtain these parameters by asking the decision maker to only provide holistic preference judgment -- should some (fictitious) alternative be assigned to some category -- and builds the shortest questionnaire.

We opt for a more indirect setup, close to a machine learning paradigm, where a set of reference assignment is given and assumed to describe the decision maker's point of view, and the aim is to \emph{extend} these assignments with a NCS model. In this context, we usually refer to an \emph{assignment} as a function mapping a subset of \emph{reference alternatives} $\ \X^\star \subset \X\ $ to the ordered set of classes $C^1 \prec \dots \prec C^p$. These reference alternatives highlight values of interest on each criterion $i\in \mathcal{N},\ \X^\star_i := \bigcup_{x\in\X^\star} x_i $. We are looking for suitable preference parameters specifying a non compensatory sorting model, i.e. a tuple of profiles $\langle b \rangle$ satisfying \eqref{eq : increasing profiles} and an upset of coalitions $\T\subset2^\N$ (respectively, non-negative voting parameters $(\langle w \rangle,\lambda)$ of a majority rule) so that U-NCS$_{(\langle b \rangle,\T)}$ (respectively, MR-Sort$_{(\langle b \rangle,\langle w \rangle,\lambda)}$) maps all reference alternatives $x\in\X^\star$ to its assigned class $A(x)$. 

Throughout this paper, we assume this expression of preference is free of noise. We are only interested in determining if the given assignment can be represented in the non compensatory sorting model.

\section{SAT formulation for learning \NCS}
\label{section : sat}

In this section, we begin by giving a brief reminder of some key concepts regarding boolean satisfiability problems (SAT). Then, we proceed by describing the pivotal contribution of this work: the encoding of the problem of representing a given assignment in the U-NCS model as a SAT problem. We conclude this section by providing the decoding procedure that prove this SAT formulation is equivalent to the original problem, and can be used in the context of Algorithm \ref{algo1} together with a SAT solver.

\subsection{Boolean satisfiability (SAT)}

A boolean satisfaction problem consists in a set of boolean variables $V$ and a logical proposition about these variables $f : \{0,1\}^V \rightarrow \{0,1\}$. A solution $v^\star$ is an assignment of the variables mapped to 1 by the proposition: $f(v^\star)=1$. A binary satisfaction problem for which there exists at least one solution is \emph{satisfiable}, else it is \emph{unsatisfiable}. Without loss of generality, the proposition $f$ can be assumed to be written in conjunctive normal form: $f = \bigwedge_{c\in\mathcal{C}}c$, 
where each \emph{clause} $c\in\mathcal{C}$ is itself a disjunction in the variables or their negation $\forall c\in\mathcal{C}, \exists c^+, c^- \in \mathcal{P}(V): c = \bigvee_{v\in c^+}v \vee \bigvee_{v\in c^-}\neg v$, so that a solution  satisfies at least one condition (either positive or negative) of every clause.

The models presented hereafter make extensive use of clauses where there is only one non-negated variable (a subset of \emph{Horn clauses}): $a\vee \neg b_1 \vee \dots \vee \neg b_n$, which represent the logical implication $(b_1 \wedge \dots \wedge b_n) \Rightarrow a$.

It is known since Cook's theorem \cite{Cook1971} that the Boolean satisfiability problem is NP-complete. Consequently, unless $P=NP$, we should not expect to solve generic SAT instances quicker than exponential time in the worst case. Nevertheless, efficient and scalable algorithms for SAT have been -- and still are -- developed, and are sometimes able to handle problem instances involving tens of thousands of variables and millions of clauses in a few seconds (see e.g. \cite{moskewicz2001,sathandbook}).

\subsection{A SAT encoding of a given assignment in U-NCS}
\begin{subequations}
\begin{definition}[SAT encoding for U-NCS] Let $A : \X^\star \rightarrow C^1 \prec \dots \prec C^p$ an assignment. We define the boolean function $\phi^\textit{SAT}_A$ with variables:
\begin{itemize}\item $x_{i,h,k}$, indexed by a criterion $i\in \mathcal{N}$, a frontier between classes $1\le h \le p-1$, and a value $k\in \X^\star_i$ taken on criterion $i$ by a reference alternative, 
\item $y_B$ indexed by a coalition of criteria $B\subseteq \mathcal{N}$  
\end{itemize}
as the conjunction of clauses:
\begin{enumerate}[label=(3\alph*):]
\item For all criteria $i\in \mathcal{N}$, for all frontiers between adjacent classes $1\le h \le p-1$, for all ordered pairs of values $k < k' \in \X^\star_i$: 
\begin{equation} x_{i,h,k'} \vee \neg x_{i,h,k}  \label{clauses : ascending scales}\end{equation}
\item For all criteria $i\in \mathcal{N}$, for all ordered pairs of frontiers $1\le h < h' \le p-1$, for all values $k\in \X^\star_i$ : 
\begin{equation} x_{i,h,k} \vee \neg x_{i,h',k} \label{clauses : profiles hierarchy}\end{equation}
\item For all ordered pairs of coalitions $B\subset B' \subseteq \mathcal{N}$: 
\begin{equation} y_{B'} \vee \neg y_B \label{clauses : coalitions strength}\end{equation}
\item For all coalitions $B\subseteq \mathcal{N}$, for all frontiers $1 \le h \le p-1$, for all $u\in \X^\star : A(u)=C^{h-1}$ (i.e. reference alternatives just below the frontier)  : 
\begin{equation} (\bigvee_{i\in B} \neg x_{i,h,u_i}) \vee \neg y_B \label{clauses : weak alternatives}\end{equation}
\item For all coalitions $B\subseteq \mathcal{N}$, for all frontiers $1 \le h \le p-1$, for all $a\in \X^\star: A(a)=C^h$ (i.e. reference alternatives just above the frontier) : 
\begin{equation} (\bigvee_{i\in B}  x_{i,h,a_i}) \vee y_{\mathcal{N}\setminus B} \label{clauses : strong alternatives}\end{equation}

\end{enumerate}
\label{def : sat}
\end{definition}
\end{subequations}
Clauses of types \eqref{clauses : ascending scales}, \eqref{clauses : profiles hierarchy} and \eqref{clauses : coalitions strength} are easily interpreted as enforcers of some monotonicity conditions inherent to ordinal sorting and to the parameters of the U-NCS model:
\begin{enumerate}[label=(3\alph*):]
\item \textit{Ascending scales} -- if $k<k'\in\X^\star_i$ and $k$ is above some threshold $\frontier^{h}_i$, then so is $k'$. It is necessary and sufficient to consider the clauses where $k$ and $k'$ are consecutive values of $\X_i^\star$.
\item \textit{Hierarchy of profiles} -- if $1\le h<h' \le p-1$ and $k\in\X_i^\star$ is above the threshold $\frontier^{h'}_i$, then it is also above $\frontier^{h}_i$. It is necessary and sufficient to consider the clauses where $h'=h+1$.
\item \textit{Coalitions strength} -- if a coalition $B\subseteq \mathcal{N}$ is sufficient, then any coalition $B' \supseteq B$ containing $B$ is also sufficient. It is necessary and sufficient to consider the clauses where the coalition $B'$ contains exactly one more criterion than $B$, corresponding to the edges represented on Fig. \ref{fig : Terry criteria}.
\end{enumerate}

Clauses of types \eqref{clauses : weak alternatives} and \eqref{clauses : strong alternatives} ensure the correct representation of all reference alternatives contained by the assignment $A$ in the U-NCS model. They rely on the following lemmas.

\begin{lemma}
Let $A : \X^\star \rightarrow C^1 \prec \dots \prec C^p$ an assignment extended by a U-NCS model with profiles $\langle \frontier \rangle$ and sufficient coalitions $\mathcal{T}$.
If $B\subseteq \mathcal{N}$ is a coalition of criteria such that, there is an  alternative $x\in\X^\star$ stronger than the upper frontier of its class $\frontier^{A(x)+1}$ on every criterion in $B$, then this coalition is not sufficient. $$\forall B\subseteq \mathcal{N},\ [\exists x\in\X^\star\ :\  \forall i \in B,\ x_i \ge \frontier_i^{A(x)+1}] \Rightarrow B\notin \mathcal{T}$$
\label{good below}
\end{lemma}
\begin{proof}Let $A$ an assignment, $(\langle \frontier \rangle,\mathcal{T})$ correct U-NCS parameters , $B$ a coalition of criteria and $x$ an alternative that satisfy the premises, and suppose $B$ is sufficient. The alternative $x$ would be better than the boundary $\frontier^{A(x)+1}$ and so would be assigned to a class strictly better than $A(x)$, and the NCS model with parameters $b$ and $\mathcal{T}$ would not extend the assignment.
\end{proof}

Clauses of type \eqref{clauses : weak alternatives} leverage Lemma \ref{good below} to ensure \textit{alternatives are outranked by the boundary above them}, relating insufficient coalitions to the strong points of an alternative.

\begin{lemma}Let $A : \X^\star \rightarrow C^1 \prec \dots \prec C^p$ an assignment extended by a U-NCS model with profiles $\langle \frontier \rangle$ and sufficient coalitions $\mathcal{T}$.
 \label{bad above}If $B\subseteq \mathcal{N}$ is a coalition of criteria such that, there is an  alternative $x\in\X^\star$ weaker than the lower frontier of its class $\frontier^{A(x)}$ on every criterion in $B$, then the complementary coalition is sufficient. $$\forall B\subseteq \mathcal{N},\ [\exists x\in\X^\star \ :\ \forall i \in B,\ x_i < \frontier_i^{A(x)}] \Rightarrow (\mathcal{N}\setminus B)\in \mathcal{T}$$
\end{lemma}

\begin{proof}Let $A$ an assignment, $(\langle \frontier \rangle,\mathcal{T})$ correct U-NCS parameters, $B$ a coalition of criteria and $x$ an alternative that satisfy the premises, and suppose $\mathcal{N}\setminus B$ is insufficient. The set of criteria on which the alternative $x$ would be better than the boundary $b^{A(x)}$ is a subset of $\mathcal{N}\setminus B$, and would thus be considered insufficient. Hence, $x$ would be assigned to a class strictly worse than $A(x)$, and the NCS model with parameters $b$ and $\mathcal{T}$ would not extend the assignment.
\end{proof}

Clauses of type \eqref{clauses : strong alternatives} leverage Lemma \ref{bad above} to ensure \textit{alternatives outrank the boundary below them}, relating the weak points of an alternative to a complementary insufficient coalition.

We are now able to describe the decoding function required by Algorithm \ref{algo1} and prove the faithfulness of both the encoding and the decoding.

\subsection{Faithfulness of the SAT representation}

\begin{theorem}[from a U-NCS model representing an assignment to a solution of the SAT formulation]
\label{th : ncs to sat}
Given an assignment $A:\X^\star \rightarrow C^1 \prec \dots \prec C^p$, if the tuple of profiles $\langle \frontier \rangle$ satisfies \eqref{eq : increasing profiles}, the set $\mathcal{T}$ is an upset of coalitions of criteria, and the sorting function U-NCS$_{\langle b \rangle,\T}$ extends $A$, then the binary tuple:\begin{itemize}\item $x_{i,h,k}$, indexed by a criterion $i\in \mathcal{N}$, a frontier between classes $1\le h \le p-1$, and a value $k\in \X^\star_i$ taken on criterion $i$ by some reference alternative, 
and defined by $x_{i,h,k} = \begin{cases} 1,& \ \textit{if} \ k \ge \frontier^h_i \\ 0, & \textit{else}\end{cases}$ 
\item $y_B$ indexed by a coalition of criteria $B\subseteq \mathcal{N}$ and defined by $y_B=\begin{cases} 1,& \ \textit{if} \ B \in \mathcal{T} \\ 0, & \textit{else}\end{cases}$ 
\end{itemize}
is mapped to 1 by the Boolean function $\phi^\textit{SAT}_A$.

\end{theorem}

\begin{proof}The clauses (\ref{clauses : ascending scales}) are satisfied because if $k<k'$ and $k$ is above some threshold $b^h$, then so is $k'$. The clauses (\ref{clauses : profiles hierarchy}) are satisfied because the frontier profiles $\langle b \rangle$ are assumed to satisfy \eqref{eq : increasing profiles} (hence, if a given value is above some threshold $\frontier^{h'}_i$, then it is also above inferior thresholds  $\frontier^h_i$ for $h<h'$). The clauses (\ref{clauses : coalitions strength}) are satisfied because $\T$ is assumed to be an upset (hence, if a coalition is deemed sufficient, then so are wider coalitions). If the \NCS model with profiles $\frontier^h$ and sufficient coalitions $T$ extends the given assignments, then clauses (\ref{clauses : weak alternatives}) are satisfied -- else, by Lemma \ref{good below}, one of the alternative $u\in \X^\star$ assigned to the class $C^{h-1}$ would outrank the profile $\frontier^h$ on a sufficient coalition of criteria - and so are clauses (\ref{clauses : strong alternatives}) - else, by Lemma \ref{bad above}, one alternative $a\in \X^\star$ assigned to class $C^{h}$ would not outrank the profile $\frontier^h$, as the set of criteria on which $a$ is better than $\frontier^h$ would be smaller than some insufficient coalition. \end{proof}

\begin{corollary}[Faithful encoding] Let $A$ be an assignment $A:\X^\star \rightarrow C^1 \prec \dots \prec C^p$. If $\phi^\textit{SAT}_A$ is unsatisfiable, then $A$ cannot be represented in the model U-NCS.

\end{corollary}


\begin{theorem}[Decoding a solution of the SAT formulation into a U-NCS model] Given an assignment $A:\X^\star \rightarrow C^1 \prec \dots \prec C^p$, if the binary tuple:
\begin{itemize}\item $x_{i,h,k}$, indexed by a criterion $i\in \mathcal{N}$, a frontier between classes $1\le h \le p-1$, and a value $k\in \X^\star_i$ taken on criterion $i$ by a reference alternative, 
\item $y_B$ indexed by a coalition of criteria $B\subseteq \mathcal{N}$  
\end{itemize}
satisfies $\phi^\textit{SAT}_A(x,y)=1$
, then the profiles  $\langle \frontier \rangle$  defined by $\frontier^h_i := \min \{k\in\X^\star_i : x_{i,h,k} = 1\}$ satisfy \eqref{eq : increasing profiles}, the set of coalitions $\mathcal{T}:=\{B\subseteq \mathcal{N} : y_B=1\}$ is an upset, and the sorting function U-NCS$_{(\langle b \rangle,\T)}$, extends the assignment $A$.
\label{th : sat to ncs}
\end{theorem}

\begin{proof} Clauses (\ref{clauses : ascending scales}) ensure that $k' \ge k \Rightarrow x_{i,h,k'}\ge x_{i,h,k}$, so that $x_{i,h,k}=1 \iff k \ge b_i^h $. Clauses (\ref{clauses : profiles hierarchy}) ensure the tuple of profiles $\langle b \rangle$ satisfies \eqref{eq : increasing profiles}. Clauses (\ref{clauses : coalitions strength}) ensure the set $\mathcal{T}$ is an upset of coalitions. The sorting function U-NCS$_{(\langle b \rangle,\T)}$ extends the given assignment because, for each reference alternative $s\in\X^\star$ , there is a clause (\ref{clauses : strong alternatives}) that ensures $s$ outranks the lower frontier of its class (if $A(s) \succ C^1$), and a clause (\ref{clauses : weak alternatives}) that ensures $s$ does not outrank the upper frontier of its class (if $A(s) \prec C^p$).
\end{proof}

\begin{corollary}[Faithfullness of the SAT representation] The assignment $A$ can be represented in the model U-NCS if, and only if, $\phi^\textit{SAT}_A$ is satisfiable.
\end{corollary}

\section{Learning MR-Sort using Mixed Integer Programming}
\label{section : MIP learning MR-SORT}

Learning the parameters of an MR-Sort model using mixed integer programming has been studied in \cite{leroy2011ADT}. We recall here the method used in \cite{leroy2011ADT} in order to obtain
the mixed inter program (MIP) formulation that infers an MR-Sort model on the basis of examples of assignments.

With MR-Sort, the condition for an alternative $a \in \X^\star$ to be assigned to a category $C^h$ (Equation \eqref{eq : majority rule}) reads:
\begin{equation*}
	a \in C^h \iff
	\begin{cases}
	\sum_{i = 1}^n c_{a,i}^{h-1} & \geq \lambda
        \\
	\sum_{i = 1}^n c_{a,i}^h & < \lambda
	\end{cases}
		\quad \text{with }\ c_{a,i}^{k} =
       	\begin{cases}
		w_i & \text{if } a_i \geq b^{k}_i,\\
		0 & \text{otherwise.}
		\end{cases}
\end{equation*}
The linearization of these constraints induces the use of binary variables. For each variable $c_{a,i}^{k}$, with $k = \{h-1, h\}$, we introduce a binary variable $\delta_{a,i}^k$ that is equal to $1$ when the performance of $a \in \X^\star$ is at least as good as or better than the performance of $b^k$ on the criterion $i$ and 0 otherwise.
For an alternative $a$ assigned to a category $C^h$ with $2 \leq h \leq p-1$, it introduces $2n$ binary variables. For alternatives assigned to one of the extreme categories, the number 	of binary variables is divided by two. The value of each variable $\delta_{a,i}^k$ is obtained thanks to the following constraints:
\begin{subequations}
\begin{equation}
	M (\delta_{a,i}^k - 1) \leq a_i - b_{i}^k < M \cdot \delta_{a,i}^k
    \label{mipconstraints:accept-reject}
\end{equation}
in which $M$ is an arbitrary large positive constant.
The value of $c_{a,i}^k$ are finally obtained thanks to the following
constraints:
\begin{equation}
	\left\{\begin{array}{rcl}
	0 \le & c_{a,i}^k & \leq w_i,\\
	\delta_{a,i}^k - 1 + w_i \le & c_{a,i}^k & \leq \delta_{a,i}^k .
	\end{array}\right.
     \label{mipconstraints:votes}
\end{equation}

\noindent The dominance structure on the set of profiles is ensured by the following constraints:
\begin{equation}
	\forall i \in \N, h = \{2, \ldots, p-1\},\ b^h_i \geq b^{h-1}_i \\	
     \label{mipconstraints:dominance}
\end{equation}
As the equation \eqref{eq : majority rule} defining the majority rule is homogenous, the coefficients $\langle w \rangle$ and $\lambda$ can be multiplied by any positive constant without modifying the upset of coalitions they represent. Thus, the following normalization constraint can be added without loss of generality:
\begin{equation}
	\sum_{i = 1}^n w_i = 1.
     \label{mipconstraints:normalize}
\end{equation}

To obtain a MIP formulation, the next step consists to define an objective function. In \cite{leroy2011ADT}, two objective functions are considered, one of which consists in maximizing the robustness of the assignments. It is done by adding continuous variables $x_a, y_a \in \mathbb{R}$ for each alternative $a \in \X^\star$ such that:
\begin{equation}
	\begin{cases}
	\sum_{i=1}^n c_{a,i}^{h-1} & = \lambda + x_a,\\
	\sum_{i=1}^n c_{a,i}^{h} & = \lambda - y_a.
	\end{cases}
     \label{mipconstraints:partial slack}
\end{equation}
The objective function consists in optimizing a slack variable $\alpha$ that is constrained by the values of the variables $x_a$ and $y_a$ as follows:
\begin{equation} \forall a\in\X^\star,
	\begin{cases}
	\alpha & \leq x_a,\\
	\alpha & \leq y_a.\\
	\end{cases} 
     \label{mipconstraints:slack}
\end{equation}
The combination of the objective function and all the constraints listed above leads to MIPs that can be found in \cite{leroy2011ADT}.

\begin{definition}[MIP-O formulation for MR Sort] Given an assignment $A$, 
we denote $\phi^{MIP-O}_A$ the mixed linear program with decision variables $\alpha$, $\lambda$, $\langle b_i^k \rangle_{i\in\N, k\in [1 .. p-1]}$, $\langle w_i \rangle_{i\in\N}$, $\langle c^h_{a,i} \rangle_{i\in\N, a\in\X^\star, h\in\{A(a)-1,A(a)\}}$, $\langle x_a \rangle_{a\in\X^\star}$, $\langle y_a \rangle_{a\in\X^\star} \in \mathbb{R}^+$ and \newline $\langle \delta^h_{a,i} \rangle_{i\in\N, a\in\X^\star, h\in\{A(a)-1,A(a)\}}\in\{0,1\}$, consisting in minimizing the objective $\alpha$, subject to the constraints (\ref{mipconstraints:accept-reject}), (\ref{mipconstraints:votes}), (\ref{mipconstraints:dominance}), (\ref{mipconstraints:normalize}), (\ref{mipconstraints:partial slack}) and (\ref{mipconstraints:slack}).

\end{definition}

\begin{theorem}[Faithfulness of the MIP-O formulation \cite{leroy2011ADT}] An assignment $A$ can be represented in the model MR-Sort if, and only if, $\phi^\textit{MIP-O}_A$ is feasible. If the tuple $\langle \alpha$, $\lambda$, $b$, $w$, $c$, $x$, $y$, $\delta \rangle$ is a feasible solution of $\phi^\textit{MIP-O}_A$, then the tuple of profiles $b$, the tuple of voting powers $w$ and the majority threshold $\lambda$ are suitable parameters of a MR-Sort model that extends the assignment $A$.
\label{th : mipo faithfulness}
\end{theorem}

We are looking to compare this state-of-the-art formulation to the boolean satisfiability formulation we propose in the next section in terms of computational efficiency, and in terms of quality of the result. Yet, we suspect the two approaches differ in too many aspects to be meaningfully compared. The MIP-O formulation is based on a numerical representation of the problem, considers the set of every MR-Sort model extending the assignment, and selects the best according to the objective function -- here, returning the model that gives the sharpest difference in voting weights between sufficient and insufficient coalitions of criteria. Meanwhile, the SAT formulation is based on a logical representation of the problem, considers the wider set of every U-NCS model extending the assignment, and randomly yields a suitable model. In order to be able to credit the effects we would observe to the correct causes, we introduce a third formulation, called MIP-D, that helps bridging the gap between MIP-O and SAT. MIP-D is formally a mixed integer program with a null objective function. This trick enables us to use the optimization shell of the MIP formulations to express a decision problem assessing the satisfiability of the constraints, and yielding a random solution (which, in our context, represents a particular MR-Sort model), rather than looking for the best one in the sense of the objective function. Another instance of this configuration, where an optimization problem is compared to its feasibility version, can be found in \cite{procaccia2014}. Here, it should be noted that the MIP-D formulation is not exactly the feasibility version of MIP-O, as insufficient coalitions of criteria are characterized by a strict comparison. The optimization version circumvents this obstacle by maximizing the contrast in normalized voting power between sufficient and insufficient coalitions, while the feasibility version addresses it by leaving the total weight unconstrained, but requiring the minimal difference between sufficient and insufficient coalitions is at least one vote. This slight difference might account for some divergence of behavior we observe during our experiment (see Section \ref{section:implem}, and particularly \ref{subsection : sat vs mip}).

\begin{definition}[MIP-D formulation for MR Sort]
We denote MIP-D the mixed linear program with decision variables  $\langle b_i^k \rangle_{i\in\N, k\in [1 .. p-1]}$, $\langle w_i \rangle_{i\in\N}$, $\lambda$, $\langle x_a \rangle_{a\in\X^\star}$, $\langle y_a \rangle_{a\in\X^\star}$, \(\langle c^h_{a,i} \rangle_{i\in\N, a\in\X^\star, h\in\{A(a)-1,A(a)\}}\) $\in \mathbb{R}^+$ and $\langle \delta^h_{a,i} \rangle_{i\in\N, a\in\X^\star, h\in\{A(a)-1,A(a)\}}\in\{0,1\}$, consisting in minimizing the objective 0, subject to the constraints (\ref{mipconstraints:accept-reject}), (\ref{mipconstraints:votes}), (\ref{mipconstraints:dominance}),  (\ref{mipconstraints:partial slack}) and (\ref{mipconstraints:bigslack}), where:
\begin{equation}
	\forall a\in\X^\star, \begin{cases}
	1 & \leq x_i,\\
	1 & \leq y_i.\\
	\end{cases}
     \label{mipconstraints:bigslack}
\end{equation}

\end{definition}
\end{subequations}
\begin{theorem}[Faithfulness of the MIP-D formulation] An assignment $A$ can be represented in the model MR-Sort if, and only if, $\phi^\textit{MIP-D}_A$ is feasible. If the tuple $\langle \lambda$, $b$, $w$, $c$, $x$, $y$, $\delta \rangle$ is a feasible solution of $\phi^\textit{MIP-D}_A$, then the tuple of profiles $b$, the tuple of voting powers $w$ and the majority threshold $\lambda$ are suitable parameters of a MR-Sort model that extends the assignment $A$.

\end{theorem}

\begin{proof} This theorem results from Theorem \ref{th : mipo faithfulness}, with only minor changes to the constraints. As noted previously, the normalization constraint \eqref{mipconstraints:normalize} has no effect on the feasibility
of the problem. Instead, constraints \eqref{mipconstraints:bigslack} ensure we are looking for voting parameters large enough to have at least a difference of one unit between the votes gathered by any sufficient coalition on the one hand and any insufficient coalition on the other hand.
\end{proof}


\section{Implementation}
\label{section:implem}
In this section, we study the performance of the formulation proposed in section \ref{section : sat}, both intrinsic and comparative with respect to state-of-the-art techniques. We implement Algorithm \ref{algo1}, using a state-of-the-art SAT solver, in order to solve instances of the problem of learning a U-NCS model, given the assignment of a set of reference alternatives. We also implement two formulations relying on Mixed Integer Programming, presented in Section \ref{section : MIP learning MR-SORT}, using an adequate solver.  We begin by describing our experimental protocol, with some implementation details.  Then, we provide the results of the experimental study concerning computation time of our program, and particularly the influence of learning set size, the number of criteria, and the number of classes, as well as elements of comparison between the three approaches.


\subsection{Experimental protocol and implementation details}
The algorithm we test takes as an input the assignment of a set of alternatives $\X^\star$, each described by a performance tuple on a set of criteria $\mathcal{N}$, to a set of classes $C^1 \prec \dots \prec C^p$. 

The performance of the solvers needs to be measured in practice, by solving actual instances of the problem and reporting the computation time required.
This experimental study is run on an ordinary laptop with Windows 7 (64 bit) equipped with an Intel Core i7-4600 CPU at 2.1 GHz and 8 GB of RAM. 

\subsubsection*{Dataset generation.}
In the scope of this paper, we only consider to use a carefully crafted, random dataset as an input. On the one hand, the algorithm we describe is not yet equipped with the capability to deal with noisy inputs, so we do not consider feeding it with actual preference data, such as the one found in preference learning benchmarks \cite{preflearning}. On the other hand, using totally random, unstructured, inputs makes no sense in the context of algorithmic decision. In order to ensure the preference data we are using makes sense, we use a decision model to generate it, and, in particular, a model compatible with the non compensatory stance we are postulating. Precisely, we use a MR-Sort model for generating the learning set, a model that particularize NCS and U-NCS by postulating the set of sufficient coalitions possess an additive structure (see Section \ref{subsection : ncs models}). 
This choice ensures the three formulations we are using should succeed in finding the parameters of a model extending the reference assignment.

When generating a dataset, we consider the number of criteria $|\mathcal{N}|$, the number of classes $p$, and the number of reference alternatives $|\X^\star|$ as parameters. We consider all criteria take continuous value in the interval $[0,1]$, which is computationally more demanding for our algorithm than the case where one criterion has a finite set of values. We generate a set of ascending profiles $\langle b \rangle$ by uniformly sampling $p-1$ numbers in the interval $[0,1]$ and sorting them in ascending order, for all criteria. We generate voting weights $\langle w \rangle$ by sampling $|\mathcal{N}|-1$ numbers in the interval $[0,1]$, sorting them, and using them as the cumulative sum of weights. $\lambda$ is then randomly chosen with uniform probability in the interval $]0.5,1[$. Finally, we sample uniformly $|\X^\star|$ tuples in $[0,1]^\mathcal{N}$, defining the performance table of the reference alternative, and assign them to classes in $C^1 \prec \dots \prec C^p$ according to the model $\mathcal{M}^0:=$MR-Sort$_{\langle b \rangle,\langle w \rangle,\lambda}$ with the generated profiles, voting weights, and qualified majority threshold.



\subsubsection*{Solving the SAT problem.}
We then proceed accordingly to Algorithm \ref{algo1}, translating the assignment into a binary satisfaction problem, described by sets of variables and clauses, as described by 
Definition \ref{def : sat}. This binary satisfaction problem is written in a file, in DIMACS format\footnote{\url{http://www.satcompetition.org/2009/format-benchmarks2009.html}}, and passed to a command line SAT solver - CryptoMiniSat 5.0.1 \cite{cryptominisat},  winner of the incremental track at SAT Competition 2016 \footnote{ \url{http://baldur.iti.kit.edu/sat-competition-2016/}}, released under the MIT license. If the solver finds a solution, then it is converted into parameters $(\langle b^{\mathrm{SAT}} \rangle,\T^{\mathrm{SAT}})$ for a U-NCS model, as described by Theorem \ref{th : sat to ncs}. The model $\mathcal{M}^{\mathrm{SAT}}:=$U-NCS$_{\langle b^{\mathrm{SAT}} \rangle,\T^{\mathrm{SAT}}}$ yielded by the program is then validated against the input. As the ground truth $\M^0$ used to seed the assignment is, by construction, a MR-Sort model and therefore a U-NCS model, Theorem \ref{th : ncs to sat} applies and we expect the solver to always find a solution. Moreover, as Theorem \ref{th : sat to ncs} applies to the solution yielded, we expect the U-NCS model returned by the program should always succeed at extending the assignment provided.

\subsubsection*{Solving the MIP problems.} We transcribe the problem consisting of finding a MR-Sort model extending the assignment with parameters providing a good contrast into a mixed integer linear optimization problem described extensively in Section \ref{section : MIP learning MR-SORT} that we refer to as \emph{MIP-O}, where O stands for \emph{optimization}. In order to bridge the gap between this optimization stance and the boolean satisfiability approach that is only preoccupied with returning any model that extends the given assignment, we also transcribe the problem consisting of finding \emph{some} MR-Sort model extending the assignment into a MIP feasibility problem (optimizing the null function over an adequate set of constraints), also described in Section \ref{section : MIP learning MR-SORT} that we refer to as \emph{MIP-D}, where D stands for \emph{decision}. These MIP problems are solved with Gurobi 7.02, with factory parameters except for the cap placed on the number of CPU cores devoted to the computation (two), in order to match a similar limitation with the chosen version of the SAT solver. When the solver succeeds in finding a solution before the time limit -- set to one hour -- the sorting functions returned are called $\mathcal{M}^{\mathrm{MIP-O}}$ and $\mathcal{M}^{\mathrm{MIP-D}}$, respectively.

\subsubsection*{Evaluating the ability of the inferred models to restore the original one.}
In order to appreciate how ``close'' a computed model $\mathcal{M}^c\in\{\mathcal{M}^{\mathrm{SAT}}$,$\mathcal{M}^{\mathrm{MIP-D}}$,$\mathcal{M}^{\mathrm{MIP-O}}\}$ is to the ground truth from which the assignment example were generated ${\cal M}^0$, we proceed as follows: we sample a large set of $n$ profiles in $\X=[0,1]^\mathcal{N}$ and compute the assignment of these profiles according to the original and computed MR Sort models (${\cal M}^0$ and $\mathcal{M}^c$). On this basis, we compute $err-rate$ the proportion of ``errors'', \textit{i.e.} tuples which are not assigned to the same category by both models. To obtain a reasonable sample for $\X$, we vary size of the sample of $\X=[0,1]^\mathcal{N}$ according to the number of criteria $|\mathcal{N}|$: $n=Max(Min( 4^{|\mathcal{N}|},3.10^5),10^4))$ 



\subsection{Intrinsic performance of the SAT formulation}
We run the experimental protocol described above by varying the various values of the parameters governing the input. In order to assess the intrinsic performance of Algorithm \ref{algo1} we consider all the combinations where\begin{itemize}

\item the number of criteria $|\mathcal{N}|$ is chosen among $\{5,7,9,11,13\}$;
\item the number of reference alternatives $|\X^\star |$ is chosen among $\{$16, 32, 64, 128, 256, 512$\}$.
\item the number of categories $p$ is chosen among $\{2,3\}$
\end{itemize}
For each value of the triplet of parameters, we sample 100 MR-Sort models $\M^0$, and record the computation time ($t$) needed to provide a model $\M^{\mathrm{NCS}}$

\begin{figure}[ht!]
\begin{center}
\includegraphics[width=\textwidth]{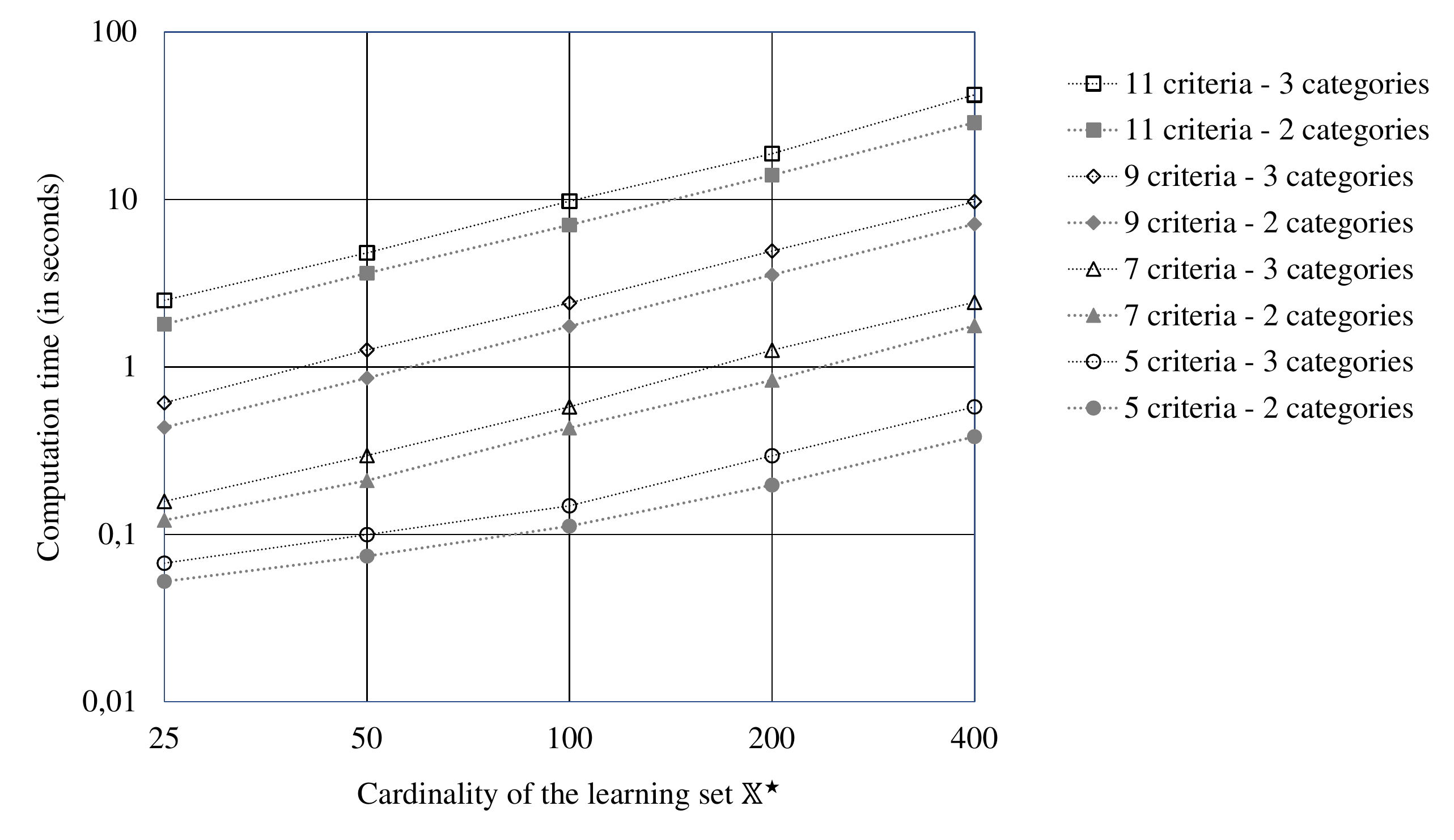}
\end{center}
\caption{Computation time by size of the learning set}
\label{Fig:computLearning}
\end{figure}

Figure \ref{Fig:computLearning} displays the time needed by Algorithm \ref{algo1} to compute $\M^{NCS}$, versus the number of reference alternatives $|\X^\star|$, both represented in logarithmic scale, in various configurations of the number of criteria. The fact that each configuration is seemingly represented by a straight line hints at a linear dependency between $\log t^{SAT}$ and $\log |\X^\star|$. The fact that the various straight lines, corresponding to various number of criteria, seem parallel, with a slope close to 1, is compatible to a law where $t_{SAT}$ is proportional to $|X^\star|$. The same observations in the plane (number of criteria $\times$ computation time) (not represented) leads to infer a law $$t^{SAT} \ \propto \ |\X^\star| \times 2^{|\mathcal{N}|}, $$ where the computing time is proportional to the number of reference alternatives and to the number of coalitions (corresponding to the number of $|\mathcal{N}|$-ary clauses of the SAT formulation). Finally, as a rule of thumb: \emph{the average computation time is about 10 s for 11 criteria, 3 categories and 100 reference alternatives; it doubles for each additional criterion, or when the number of reference alternatives doubles}.

\subsection{Comparison between the formulations}
\label{subsection : sat vs mip}

In order to compare between models, we focus on a situation with three categories, nine criteria, and 64 reference alternatives, serving as a baseline. We then consider situations deviating from the baseline on a single parameter -- either the number of categories $p$, from 2 to 5, or the number of criteria, among $\{5,7,9,11,13\}$, or the number of reference alternative among $\{16,32,64,128,256\}$. For each considered value of the triple of parameters, we sample 50 MR Sort models representing the ground truth $\M^0$, and we record the computation time $t$ needed to provide each of the three models $\M^\mathrm{NCS}$, $\M^\mathrm{MIP-D}$ and $\M^\mathrm{MIP-O}$, as well as the generalization indexes for the three models. The MIP are solved with a timeout of one hour.


\subsubsection*{Results on the computation time.}


\begin{figure}[ht!]
\begin{center}
\includegraphics[width=\textwidth]{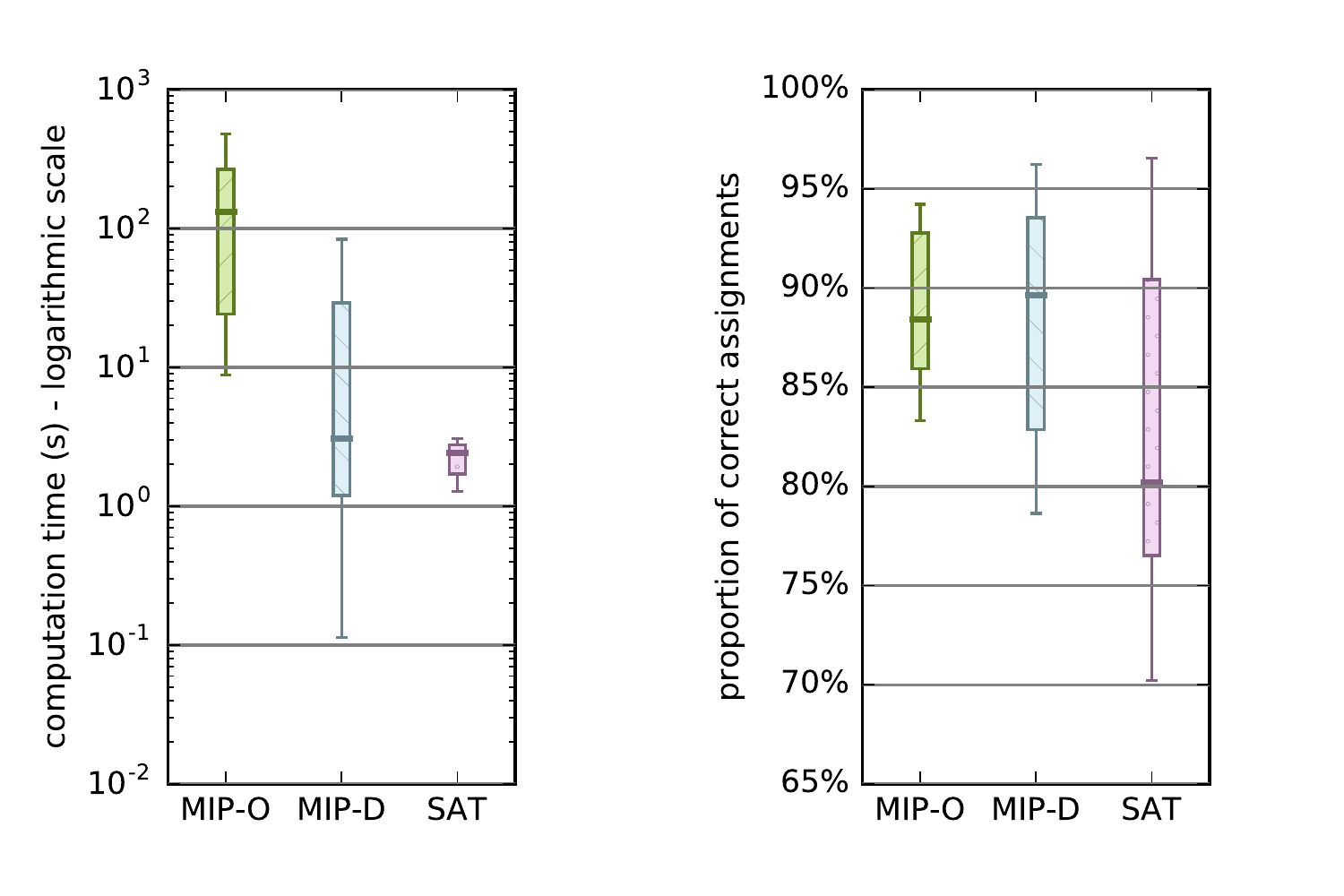}
\end{center}
\caption{Distribution of the computation time and the proportion of assignment similar to the ground truth for the three models in the baseline configuration: 9 criteria, 3 categories, 64 reference alternatives. Represented: median; box: $25-75\%$; whiskers: $10-90\%$.}
\label{Fig:timebaseline}
\end{figure}

For the three formulations under scrutiny and the set of considered parameters governing the input, the computation time ranges from below the tenth of a second to an hour (when the timeout is reached), thus covering about five orders of magnitude. The left side of Figure \ref{Fig:timebaseline} depicts the distribution of the computation time for the baseline situation (9 criteria, 3 categories, 64 reference assignments).
While the computing time for the SAT and the MIP-D formulations seem to be centered around similar values (with $Med(t^{SAT}) \approx 2.4 s$ and $Med(t^{MIP-D}) \approx 3.1 s$ for the baseline), the distribution of the computing time for the SAT algorithm around this center is very tight, while the spread of this distribution for the MIP-D formulation is comparatively huge: the slowest tenth of instances run about a thousand time slower than the quickest tenth. The computation time of the MIP-O formulation appears about 50 times slower than the SAT, with a central value of $Med(t^{MIP-O}) \approx 130 s$, and covers about two orders of magnitude.

\begin{figure}[ht!]
\begin{center}
\includegraphics[width=\textwidth]{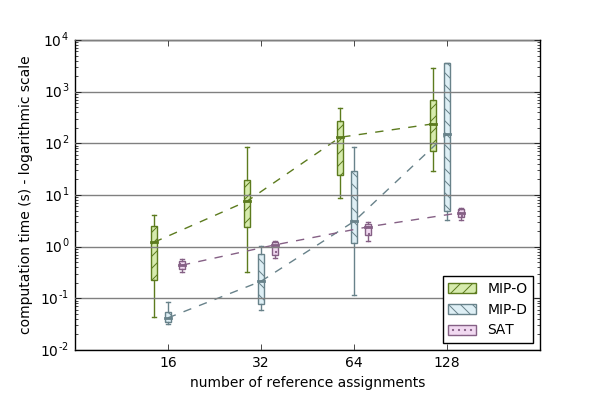}
\end{center}
\caption{Distribution of the computation time for the three models by number of reference assignments}
\label{Fig:ctime_vs_X}
\end{figure}

\begin{figure}[ht!]
\begin{center}
\includegraphics[width=\textwidth]{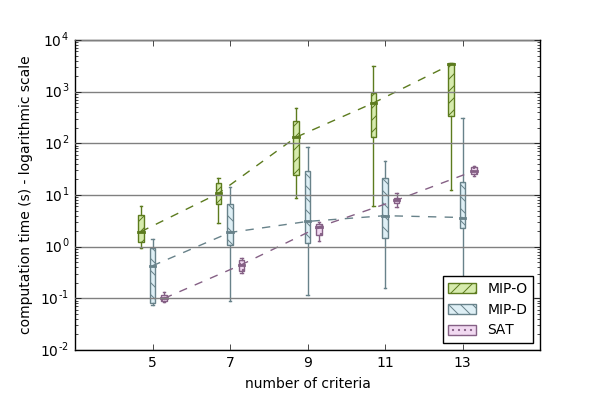}
\end{center}
\caption{Distribution of the computation time for the three models by number of criteria}
\label{Fig:ctime_vs_N}
\end{figure}

\begin{figure}[ht!]
\begin{center}
\includegraphics[width=\textwidth]{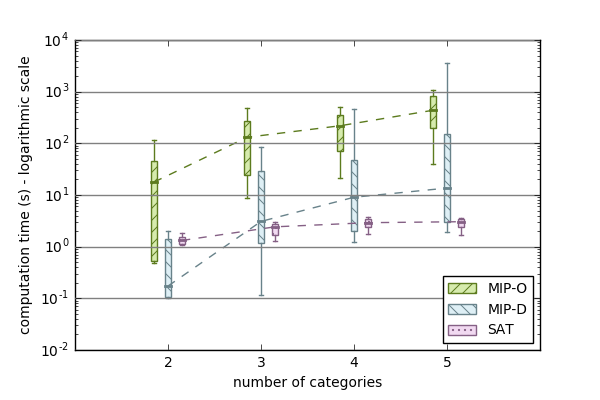}
\end{center}
\caption{Distribution of the computation time for the three models by number of categories}
\label{Fig:ctime_vs_p}
\end{figure}

In order to assess the influence of the parameters governing the size and complexity of the input, we explore situations differing from the baseline on a single parameter.\begin{itemize}
\item \emph{The number of reference assignments $|\X^\star|$}. Figure \ref{Fig:ctime_vs_X} indicates that the distribution of the computing time for SAT-based algorithm remains tightly grouped around its central value, and that this value steadily increases with the number of reference assignments. Meanwhile, the two MIP formulations display a similar behavior, with an increase of the central tendency steeper than the one displayed by the SAT, and a spread that widens when taking into account additional reference assignments.
\item \emph{the number of criteria $|\N|$}. Figure \ref{Fig:ctime_vs_N} indicates that the distribution of the computing time for SAT-based algorithm remains tightly grouped around its central value, and that this value steadily increases with the number of criteria. This increase is steeper in the case of the SAT and MIP-O formulations than for the MIP-D formulation.
\item \emph{the number of categories $p$}. Figure \ref{Fig:ctime_vs_p} displays an interesting phenomenon. The number of categories seems to have a mild influence on the computation time, without any restriction for the SAT-based algorithm, and as soon as there are three categories or more for the MIP-based algorithm, with a clear exception in the case of two categories, which yields instances of the problem solved ten times faster than with three or more categories.
\end{itemize}

\subsubsection*{Results on the ability of the inferred model to restore the original one.}

The right side of Figure \ref{Fig:timebaseline} depicts the distribution of the proportion of correct assignments (as compared to the ground truth) for the baseline situation (9 criteria, 3 categories, 64 reference assignments).
The situation depicted is conveniently described by using the distribution of outcomes yielded by the MIP-D formulation as a pivotal point to which we compare those yielded by the SAT and MIP-O formulations: the central 80\% of the distribution (between the whiskers) of outcomes for the MIP-O corresponds to the central half (the box) for the MIP-D, while the best half of the distribution of outcomes for the SAT corresponds to the central 80\% for the MIP-D. In other terms, compared to the MIP-D, the MIP-O offers consistently good results, while the SAT has a 50\% chance to yield a model that does not align very well with the ground truth.

Figures \ref{Fig:gene_vs_X}, \ref{Fig:gene_vs_N} and \ref{Fig:gene_vs_p} depict the variations of the alignment of the models yielded by the three algorithms with the ground truth with respect to the number of reference assignments, of criteria, or of categories, respectively. The experimental results display a tendency towards a degradation of this alignment as the number of criteria or the number of categories increase. Conversely, as expected, increasing the number of reference assignments noticeably enhances the restoration rate. The three algorithms seem to behave in a similar manner with respect to the modification of these parameters.



\begin{figure}[ht!]
\begin{center}
\includegraphics[width=\textwidth]{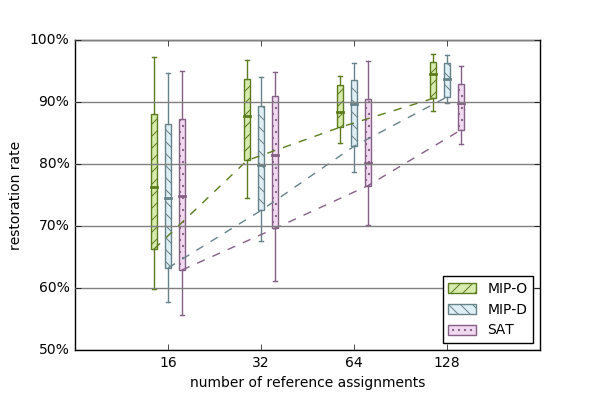}
\end{center}
\caption{Distribution of the generalization index for the three models by size of the learning set}
\label{Fig:gene_vs_X}
\end{figure}

\begin{figure}[ht!]
\begin{center}
\includegraphics[width=\textwidth]{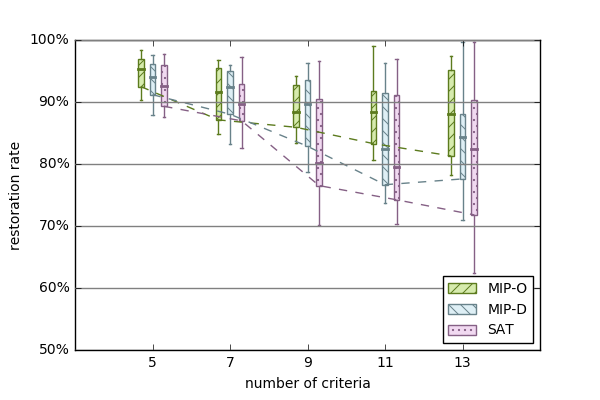}
\end{center}
\caption{Distribution of the generalization index for the three models by number of criteria}
\label{Fig:gene_vs_N}
\end{figure}

\begin{figure}[ht!]
\begin{center}
\includegraphics[width=\textwidth]{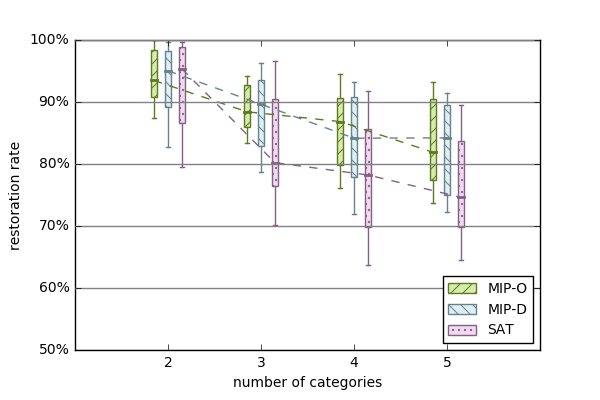}
\end{center}
\caption{Distribution of the generalization index for the three models by number of categories}
\label{Fig:gene_vs_p}
\end{figure}








\subsubsection*{Reliability.}
The three formulations expressing the problem we solve -- finding a non compensatory sorting model extending a given assignment of reference alternatives -- into technical terms are theoretically faithful. Moreover, as we generate the input assignment with a hidden \emph{ground truth} which itself obeys a non-compensatory sorting model, the search we set out to perform should provably succeed. Unfortunately, a computer program is but a pale reflection of an algorithm, as it is restricted in using finite resource. While we take great care in designing the experimental protocol in order to avoid memory problems, we have purposefully used off-the-shelf software with default setting to solve the formulations. While this attitude has given excellent result for the implementation of the SAT-based algorithm, which has never failed to retrieve a model that succeeds in extending the given assignment, the two MIP-based implementations have suffered from a variety of failures, either not terminating before the timeout set at one hour or wrongly concluding on the infeasibility of the MIP. 
 We report these abnormal behaviors in Table \ref{tab:failures}.

\begin{table}
\begin{center}
\begin{tabular}{|p{14em}|c|c|c|c|c|c|c|c|} \hline number of criteria & 5 & 7 & 9 & 11 & 13 & 7 & 9 & 9 \\  number of categories & 3 & 3 & 3 & 3 & 3 & 3 & 5 & 7 \\  number of reference assignments & 64 & 64 & 64 & 64 & 64 &128 & 64 & 64 \\ \hline  MIP-D & 4\%$^\dagger$ & 8\%$^\dagger$ & 4\% & 0 & 0 & 42\% & 10\% & 12\% \\ MIP-O & 0 & 0 & 0& 10\% & 48\% & 4\% & 0 & 0 \\ SAT & 0 & 0 & 0 & 0 & 0 & 0 & 0 & 0 \\ \hline \end{tabular}

\vskip1ex
\caption{Proportion of instances failing to retrieve a model. The default case is due to reaching the time limit, except for configurations marked with a dagger where the failure is due to an alleged infeasibility of the formulation.}\label{tab:failures}
\end{center}
\end{table}

\section{Discussion and perspectives}
\label{section:discussion}
In this section, we strive at interpreting the results presented in Section \ref{section:implem}. In Section \ref{subsection : influence parameters}, we address the influence of the parameters governing the size and structure of the input - the reference assignment we set out to extend with a non-compensatory sorting model - on the computing time of the programs implementing the three formulations modeling the problem. In Section \ref{subsection : knowledge representation}, we discuss the differing approaches to knowledge representation underlying these different formulations, and their practical consequences. 


\subsection{Influence of the parameters}
\label{subsection : influence parameters}
The influence of the various parameters ($|X^\star|$, the number of reference assignments; $|\N|$, the number of criteria; $p$, the number of categories) governing the input on the ability of the output model to predict the ground truth seeding the input is best understood from a machine learning perspective. The input assignments
form the learning set of the algorithm, while the number of criteria and the number of categories govern the number of parameters describing the non-compensatory sorting model. Hence, an increase in $|\X^\star|$ adds constraints upon the system, while increases in $|\N|$ or $p$ relieve some constraints, but demand more resources for their management.
\begin{itemize}
\item The comparison between MIP-O and MIP-D informs the influence of the loss function. This influence is threefold: optimizing this function demands a lot more time than simply returning the first admissible solution found; formalizing the problem of extending the input assignment with a model as an optimization problem incorporates a kind of robustness into the algorithm, which translates to a decrease in the number of failures; paradoxically, the strategy consisting in finding the most representative model (in the sense of the chosen loss function) does not yield models with a better alignment to the ground truth than the one consisting to return a random suitable model.

\item The MIP-D and SAT formulations implement the same binary attitude concerning the suitability of a non-compensatory model to extend a given assignment, and both arbitrarily yield the first-encountered suitable model. Nevertheless, algorithms based on these formulations display marked differences in behavior: while the running time of the SAT-based algorithm is very homogeneous between instances and follows very regular patterns when the input parameters change, the MIP-D algorithm behaves a lot more erratically, with some failures (displayed in table \ref{tab:failures}) and a tremendous spread. We credit this difference in behavior to a difference of approach to knowledge representation, as discussed in section \ref{subsection : knowledge representation}. Also, with the same input parameters, the model returned by the MIP-D algorithm seems on average to be more faithful to the ground truth than the model returned by the SAT algorithm. As both models return random suitable models in different categories (MR Sort for the MIP, and the superset NCS for the SAT, while the ground truth is chosen in the MR Sort category), we interpret the difference in the proportion of correct assignment to the respective volumes of the two categories of model, and discuss the pros and cons of assuming one or the other in section \ref{subsection : knowledge representation}.

\item Reference assignments are a necessary evil. On the one hand, they provide the information needed to entrench the model, and refine the precision up to which its parameters can be known. On the other hand, they erect a computational barrier which adds up more quickly for the MIP formulations we are considering than for the SAT one, as shown in Figure \ref{Fig:ctime_vs_X} . Overcoming this barrier demands time and threatens the integrity of the somewhat brittle numerical representation underlying the MIP-D formulation.

\item From the perspective of the model-fitting algorithm, the number of criteria and the number of categories are usually exogenous parameters, fixed according to the needs of the decision situation. The specific numbers of criteria we considered during the experiment, from 5 to 13, cover most of the typical decision situations considered in MCDA. Introducing more criteria demands to assess more parameters, which has a compound effect on complexity, as it requires at the same time to build a higher dimension representation of the models, and to provide more reference examples in order to be determined with a precision suitable to decision making. Apart from a noticeable exception (see below), the number of categories does not seem to have much influence (as shown on figures \ref{Fig:ctime_vs_p} and \ref{Fig:gene_vs_p}).


\item Underconstrained models are not very good at providing recommendations. When fed with scarce information, the task of finding a suitable extension is easy, but there are very little guarantees this extension matches the unexpressed knowledge and preferences of the decision maker concerning alternatives outside the learning set. We interpret the decrease in the ability to align with the ground truth as the number of criteria increases displayed on Figure \ref{Fig:gene_vs_N} as an expression of an \emph{overfitting} phenomenon, where too many parameters are chosen to faithfully represent a too little slice of the set of alternatives, but poorly represent cases never seen before.

\item Mixed integer programs can represent decision problems, in theory. Practically though, some complex inputs have proven overwhelming for the MIP-D formulation, whereas the MIP-O has shown more robustness, as evidenced by Table \ref{tab:failures}. It seems fair to assume this lack of stability is related to the absence of a normalization constraint such as \eqref{mipconstraints:normalize} in the MIP-D formulation .

\item MR Sort with two categories is structurally different than models with more than two categories. While we have defined it as a procedure where alternatives are compared holistically to a profile, it can also be described as an additive value sorting model with stepwise, non-decreasing, 2-valued marginals. The experimental results, both for the computing time  and the alignment with the ground truth (see figures \ref{Fig:ctime_vs_p} and \ref{Fig:gene_vs_p}, where the points corresponding to two categories are outliers with respect to the rest of the series) highlight this peculiarity, and tend to show that the value-based representation of the  MR-Sort model with two categories is computationally efficient. Determining a good lower bound on the difference of normalized voting power between sufficient and insufficient coalitions would therefore likely help alleviating this issue.

\end{itemize}

\subsection{Numeric or symbolic representation of coalitions}
\label{subsection : knowledge representation}
Our proposal to infer non compensatory sorting models from assignment examples using a SAT formulation relies on a symbolic representation of sufficient coalitions of criteria. It departs fundamentally from the state-of-the-art approach of representing the upset of sufficient coalitions with a numeric majority rule \eqref{eq : majority rule}
.

Obviously, U-NCS is more general than MR-Sort as additive weights/majority level induce a set of minimal criteria coalitions, while a set of minimal coalitions might not be additive. 
\cite{Ersek-et-al-DAM2017} studies the proportion of additive representations: all NCS models are additive up to 3 criteria, but the proportion of additive NCS models tends to be quickly marginal when the number of criteria increases. 
It is also possible possible to extend the MR-Sort model up to U-NCS by considering a capacity instead of a weight vector (see e.g. \cite{sobrieetal2015ADT}). This leads to MIP formulations more and more awkward as the arity of the capacity increases, increasing the number of decision variables and the computation time. Also, \cite{sobrieetal2015ADT} shows that a MR-Sort model learned from NCS generated examples provides a good approximation of this NCS model.

A distinctive feature of MR-Sort is its parsimony with respect to interaction between criteria, a notion that the SAT formulation of U-NCS fails to capture. However, there are many ways to additively represent a set of minimal coalitions, and the intuitive interpretation of the weights can therefore be misleading: there is no one to one correspondence between the tuples of voting powers and majority level, and the sets of additive coalitions of criteria. For instance, consider a three criteria problem in which minimal coalitions are all combination of two criteria out of three. This set of minimal coalitions can be represented by $w=(\frac{1}{3},\frac{1}{3},\frac{1}{3})$ and $\lambda=\frac{1}{2}$, or  $w=(0.49,0.49,0.02)$ and $\lambda=\frac{1}{2}$. It is obvious that these two numerical representations yield erroneously to two very distinct interpretations about the relative importance of criteria. In this sense, the symbolic representation avoiding weights used in the SAT formulation is more faithful than a numerical representation. As a consequence, this non-uniqueness of the additive representation penalizes the effectiveness of loss functions involved in MIPs.

Also, as mentioned in Section \ref{subsection : sat vs mip}, the feasibility version of the MIP suffers from numerical instability, perhaps because of the lack of a normalization constraint. The symbolic representation of sufficient coalitions circumvents the difficult mathematical question of providing a good lower bound on the worst case difference of normalized voting power between sufficient and insufficient coalitions.

\section{Conclusion}
\label{section:conclu}

In this paper, we consider the multiple criteria non-compensatory sorting model \cite{bouyssoumarchant2007a,bouyssoumarchant2007b} and propose a new SAT formulation for inferring this sorting  model from a learning set provided by a DM. Learning this model has already been addressed by the literature, and solved by the resolution of a MIP \cite{leroy2011ADT} or via a specific heuristic \cite{sobrieADT2013,sobrieetal2015ADT}. Due to high computation time, the MIP formulation can only apply to learning sets of limited size. Heuristic methods can handle large datasets, but can not ensure to find a compatible model with the learning set whenever it exists. Our new algorithm provides such guaranty. We implemented and tested our SAT formulation, and it outperforms MIP approaches in terms of computation time (reduction of the computation time by a factor of about 50). 

Moreover, it could have been the case that this good performance in terms of computing time would be counterbalanced by a limited ability of the inferred model in terms of generalization. Indeed, a MIP approach focuses the effort in finding a relevant representative model among the compatible models (through the use of an objective function), while our SAT approach does return the first compatible model found. 

Our experiments show that MIP and SAT approaches have similar performances in terms of generalization. Therefore, we believe this algorithm to be a strong advance in terms of learning NCS models based on learning sets, in particular when learning sets become relatively large.

Thanks to its efficiency -- finding a model compatible to some preference information takes seconds instead of minutes -- this algorithm is well suited to be embedded in an interactive process, where the decision maker is invited to interactively elicit a non compensatory sorting model by incrementally building a learning set (and possibly additional preference information).
In order to address real-world decision aiding situations, the algorithm we propose needs to be equipped with techniques permitting to account for noisy or inconsistent data. While the numeric formulations may rely on Lagrangian techniques to handle the requirement of correctly representing the data as a set of soft constraints rather than hard ones, the logic formulation we propose could usefully investigate the notions of maximally consistent or minimally inconsistent set of clauses (see e.g. \cite{BesnardGL15} for solving techniques, or e.g. \cite{mousseau2003} for an application in a MCDA context).
The increased speed, as compared to the previous MIP-based approach, opens the door to the exploration of the set of all U-NCS models extending a given assignment, in the vein of the version space theory \cite{mitchell82} and robust decision aiding \cite{utagms}. The knowledge representation underlying our approach may also permit to support a recommendation with an \emph{explanation} \cite{Amgoud2008,labreuche2011,belahceneIJCAI2017}.


\begin{thebibliography}{1}

\bibitem{belahcene-et-alEJORsubmitted}
K. Belahc\`ene, V. Mousseau, M. Pirlot and O. Sobrie. Preference Elicitation and Learning in a Multiple Criteria Decision Aid perspective.  
LGI Research report 2017-02, CentraleSup\'elec, http://www.lgi.ecp.fr/Biblio/PDF/CR-LGI-2017-02.pdf, 2017.

\bibitem {bouyssoumarchant2007a}
D. Bouyssou, and T. Marchant. An axiomatic approach to noncompensatory sorting methods in MCDM, I:
The case of two categories. European Journal of Operational Research, 178(1):217–-245, 2007.

\bibitem {bouyssoumarchant2007b}
D. Bouyssou, and T. Marchant. An axiomatic approach to noncompensatory sorting methods in MCDM, II:
More than two categories. European Journal of Operational Research, 178(1):246-–276, 2007.


\bibitem{bouyssouetal2006}
D. Bouyssou, T. Marchant, M. Pirlot, A. Tsouki\`as, and Ph. Vincke. Evaluation and decision models with
multiple criteria: Stepping stones for the analyst. International Series in Operations Research and Management
Science, Volume 86. Springer, Boston, 1st edition, 2006.

\bibitem{Ersek-et-al-DAM2017}
E. Ersek Uyanik, O. Sobrie, V. Mousseau, and M. Pirlot. Enumerating and categorizing positive Boolean functions separable by a k-additive capacity. Discrete Applied Mathematics, 229, 	17-30, 2017.

\bibitem{GRECO-ET-AL2010}
S. Greco, V. Mousseau, and R. Slowinski. Multiple criteria sorting with a set of additive value functions. European Journal of Operational Research, 207(3), 1455-1470, 2010.

\bibitem{GRECO2011}
S. Greco, M. Kadzinski, and R. Slowinski. Selection of a representative value function in robust multiple criteria sorting. Computers \& Operations Research, 38(11), 1620-1637, 2011.

\bibitem{Karsu-COR2016}
\H{O}. Karsu, Approaches for inequity-averse sorting, Computers \& Operations Research, 66, 67-80, 2016.

\bibitem{leroy2011ADT}
A. Leroy, V. Mousseau, and M. Pirlot. Learning the parameters of a multiple criteria sorting method. In
R. Brafman, F. Roberts, and A. Tsouki\'as, editors, Algorithmic Decision Theory, volume 6992 of Lecture Notes in Artificial Intelligence, pages 219--233. Springer, 2011.

\bibitem{MARICHAL2005}
J.-L. Marichal, P. Meyer, and M. Roubens. Sorting multi-attribute alternatives: The TOMASO method. Computers \& Operations Research, 32(4), 861-877, 2005.

\bibitem{MEYER2017}P. Meyer, A.-L. Olteanu. Integrating large positive and negative performance differences into multicriteria majority-rule sorting models. Computers \& Operations Research, 81, 216--230, 2017.

\bibitem{roybouyssou1993}
B. Roy, and D. Bouyssou. Aide multicritère \`a la d\'ecision: m\'ethodes et cas. Economica, Paris, 1993.

\bibitem{sobrieADT2013}
O. Sobrie, V. Mousseau, and M. Pirlot. Learning a majority rule model from large sets of assignment examples. In P. Perny, M. Pirlot, and A. Tsouki\`as, editors, Algorithmic Decision Theory, pages 336-–350, Brussels, Belgium, 2013. Springer.

\bibitem{sobrieetal2015ADT}
O. Sobrie, V. Mousseau, and M. Pirlot. Learning the Parameters of a Non Compensatory Sorting Model.  In Algorithmic Decision Theory, ADT 2015, pages 153--170, 2015.

\bibitem{sobrie2017PhD}
O. Sobrie. Learning preferences with multiple-criteria models.  PhD thesis, CentraleSup\'elec and UMONS, 2017.
\bibitem{SOYLU2011}
B. Soylu. A multi-criteria sorting procedure with Tchebycheff utility function. Computers \& Operations Research, 38(8), 1091-1102, 2011.

\bibitem{Zheng-et-al-COR2014}
J. Zheng, S.A. Metchebon TakougangV. Mousseau, and M. Pirlot. Learning criteria weights of an optimistic Electre Tri sorting rule. Computers \& Operations Research, 49, 28-40, 2014.

\bibitem{gms2002}
S. Greco, B. Matarazzo and R. Slowinski. Rough sets methodology for sorting problems in presence of multiple attributes and criteria. European Journal of Operational Research; 138(2), 247 - 259, 2002

\bibitem{gms2016}
S. Greco, B. Matarazzo and R. Slowinski. Decision rule approach. in Multiple Criteria Decision Analysis - State of the art surveys, S. Greco, M. Ehrgott and J. Figueira (eds.) second edition, 497-552, 2016

\bibitem{ms98}
V. Mousseau and R. Slowinski. Inferring an ELECTRE TRI model from assignment examples, Journal of Global Optimization, 12(2), 157-174, 1998 

\bibitem{zd2002}
C. Zopounidis and M. Doumpos. Multicriteria classification and sorting methods: A literature review, European Journal of Operational Research; 138(2), 229-246, 2002.

\bibitem{sathandbook}
A. Biere, M. Heule, H. van Maaren, and T. Walsh. Handbook of Satisfiability. Frontiers in Artificial Intelligence and Applications 185, IOS Press, 2009.

\bibitem{UTA}
E. Jacquet-Lagreze and Y. Siskos. Assessing a set of additive utility functions for multicriteria decision-making, the {UTA} method. European Journal of Operational Research, 10(2), 151-164, 1982. 

\bibitem{moskewicz2001}
M. Moskewicz, C. Madigan, Y. Zhao, L. Zhang, and S. Malik. Chaff: engineering an efficient SAT solver. In Proceedings of the 38th annual Design Automation Conference (DAC '01). ACM, New York, NY, USA, 530-535. 2001.

\bibitem{cryptominisat}
M. Soos, The CryptoMiniSat 5 set of solvers at SAT Competition 2016. In Proceedings of SAT Competition 2016: Solver and Benchmark Descriptions, volume B-2016-1 of Department of Computer Science Series of Publications
B, University of Helsinki, 2016. 

\bibitem{Amgoud2008}
L. Amgoud and M. Serrurier. Agents that argue and explain classifications. In Autonomous Agents and Multi-Agent Systems, AAMAS 2008.

\bibitem{BesnardGL15}
Ph. Besnard, E. Gr{\'{e}}goire and J.{-}M. Lagniez. On Computing Maximal Subsets of Clauses that Must Be Satisfiable with Possibly Mutually-Contradictory Assumptive Contexts. Proceedings of the Twenty-Ninth {AAAI} Conference on Artificial Intelligence, January 25-30, 2015, Austin, Texas, USA.

\bibitem{mousseau2003}
V. Mousseau, J. Rui Figueira,L. C. Dias, C. Gomes da Silva and J. C. N. Cl{\'{\i}}maco. Resolving inconsistencies among constraints on the parameters of an {MCDA} model. European Journal of Operational Research
147(1), 72--93, 2003.

\bibitem{labreuche2011}
Ch. Labreuche. A general framework for explaining the results of a multi-attribute preference model. Artificial Intelligence  175(7-8), 1410--1448,
 2011.
 
\bibitem{belahceneIJCAI2017}
K. Belahcene, Ch. Labreuche, N. Maudet, V. Mousseau and W. Ouerdane.
 A Model for Accountable Ordinal Sorting. Proceedings of the Twenty-Sixth International Joint Conference on Artificial Intelligence, {IJCAI} 2017, 814--820.
 
\bibitem{mitchell82}
T. Mitchell. Generalization as Search. Artificial Intelligence 18(2) 203--226, 1982.

\bibitem{utagms}
S. Greco, V. Mousseau and Roman Slowinski. Ordinal regression revisited: Multiple criteria ranking using a set of additive value functions. European Journal of Operational Research 191(2), 416--436, 2008.

\bibitem{preflearning}
J. Furnkranz, E. Hullermeier. Preference Learning. Springer-Verlag New York, Inc, 2011.

\bibitem{Cook1971}
S. Cook. The complexity of theorem proving procedures. Proceedings of the third annual ACM symposium on Theory of computing, 151-158, 1971.

\bibitem{procaccia2014}
J. Dickerson, J. Goldman, J. Karp, A. Procaccia, T. Sandholm. The Computational Rise and Fall of Fairness. AAAI 2014, 1405-1411

\end{thebibliography}


\end{document}